\documentclass{article}

\PassOptionsToPackage{numbers, compress}{natbib}
 \usepackage[main, final]{neurips_2026}
\usepackage[utf8]{inputenc} % allow utf-8 input
\usepackage[T1]{fontenc}    % use 8-bit T1 fonts
\usepackage{hyperref}       % hyperlinks
\usepackage{url}            % simple URL typesetting
\usepackage{booktabs}       % professional-quality tables
\usepackage{amsfonts}       % blackboard math symbols
\usepackage{nicefrac}       % compact symbols for 1/2, etc.
\usepackage{microtype}      % microtypography
\usepackage{xcolor}         % colors

\usepackage{algorithm}
\usepackage{algpseudocode}

\usepackage{dblfloatfix}
\usepackage{listings} % replaces minted, which needs -shell-escape (disabled on arXiv)
\usepackage{enumitem}

\usepackage{amsmath}
\usepackage{amssymb}
\usepackage{mathtools}
\usepackage{amsthm}
\usepackage{graphicx}
\usepackage{subcaption}
\usepackage{multirow}
\usepackage{wrapfig}
\usepackage[capitalize,noabbrev]{cleveref}

\usepackage{xspace}
\makeatletter
\DeclareRobustCommand{\eg}{e.g.\@\xspace}

\makeatother

\usepackage{xcolor}
\definecolor{codebackground}{RGB}{240, 240, 235}
\definecolor{codekeyword}{rgb}{0.00,0.50,0.00}
\definecolor{codenamespace}{rgb}{0.00,0.00,1.00}
\definecolor{codestring}{rgb}{0.73,0.13,0.13}
\definecolor{codeoperator}{rgb}{0.40,0.40,0.40}
\definecolor{codecomment}{rgb}{0.24,0.48,0.48}
\lstdefinestyle{b2pcode}{
  language=Python,
  basicstyle=\ttfamily,
  columns=fullflexible,
  keepspaces=true,
  showstringspaces=false,
  keywordstyle=\color{codekeyword}\bfseries,
  emph={b2p},
  emphstyle=\color{codenamespace}\bfseries,
  stringstyle=\color{codestring},
  commentstyle=\color{codecomment}\itshape,
  literate=*{=}{{\textcolor{codeoperator}{=}}}1
            {.}{{\textcolor{codeoperator}{.}}}1
            {0.05}{{\textcolor{codeoperator}{0.05}}}4
            {256}{{\textcolor{codeoperator}{256}}}3,
  backgroundcolor=\color{codebackground},
  frame=single,
  framerule=0pt,
  rulecolor=\color{codebackground},
  framesep=\fboxsep,
  xleftmargin=\fboxsep,
  xrightmargin=\fboxsep,
  aboveskip=9.23pt plus 2pt minus 2pt,
  belowskip=11.21pt plus 2pt minus 2pt,
}

\usepackage{amsmath,amsfonts,bm}

\def\eqref#1{equation~\ref{#1}}
\def\1{\bm{1}}

\DeclareMathAlphabet{\mathsfit}{\encodingdefault}{\sfdefault}{m}{sl}
\SetMathAlphabet{\mathsfit}{bold}{\encodingdefault}{\sfdefault}{bx}{n}

\newcommand{\E}{\mathbb{E}}

\newcommand{\R}{\mathbb{R}}

\newcommand{\Var}{\mathrm{Var}}

\newcommand{\Cov}{\mathrm{Cov}}
\title{B2P-Corr: Batch-to-Population Gradient Estimators for Non-Decomposable Correlation Losses}
\author{%
  Jingquan Yan\textsuperscript{*}, Yuwei Miao\textsuperscript{*}, Peiran Yu, Junzhou Huang \\
  Department of Computer Science and Engineering \\
  University of Texas at Arlington \\
  \texttt{\{jingquan.yan,yuwei.miao,peiran.yu,jzhuang\}@uta.edu} \\
    \textsuperscript{*}Equal contribution.
}

\theoremstyle{plain}
\newtheorem{theorem}{Theorem}[section]
\newtheorem{proposition}[theorem]{Proposition}
\newtheorem{lemma}[theorem]{Lemma}

\theoremstyle{definition}

\newtheorem{assumption}[theorem]{Assumption}
\theoremstyle{remark}
\newtheorem{remark}[theorem]{Remark}

\begin{document}

\maketitle

\begin{abstract}
Correlation metrics such as the Pearson Correlation Coefficient (PCC) and Concordance Correlation Coefficient (CCC) are standard regression metrics and learning objectives, but their non-decomposable nature poses a challenge for stochastic optimization. Since these metrics depend on dataset-level moments (mean, variance, covariance), naive mini-batch training with local statistics yields biased, high-variance gradients whose expectation does not align with the population gradient. To bridge this ``batch-to-population'' (B2P) gap, we propose \textbf{B2P-Corr}, a novel framework that transforms correlation objectives into decomposable pseudo-losses driven by lightweight global ``moment sketches''. We introduce two efficient variants: \textbf{B2P-EMA}, which utilizes exponential moving averages and stop-gradient operators to ensure asymptotic gradient tracking with $O(B)$ cost ($B$ is batch size); and \textbf{B2P-U}, which uses sampled cross-batch U-statistic estimators and a reservoir-based stabilization with an $O(B)$ implementation. Theoretically, we quantify the $O(B^{-1})$ bias of naive mini-batch optimization, prove a moving-target tracking result for B2P-EMA under a fast-sketch / slow-parameter regime, and provide an idealized control-variate analysis for B2P-U. Experiments on two synthetic and three real-world datasets demonstrate B2P-Corr achieves stable optimization and higher correlation than naive mini-batch optimization with negligible overhead.
\end{abstract}

\section{Introduction}

Stochastic optimization (\eg, SGD~\citep{ruder2016overview}) are most effective for \emph{decomposable} objectives, where the population risk is an expectation of per-example losses and mini-batch gradients are unbiased estimators of the population gradient. This assumption breaks for \emph{non-decomposable} objectives whose values depend on dataset-level statistics or cross-example interactions. Well-known examples include AUC/F-measures in classification \citep{sanyal2018optimizing} and correlation-based regression objectives such as the Pearson correlation coefficient (PCC) \citep{schober2018correlation} and concordance correlation coefficient (CCC) \citep{lawrence1989concordance}.

Correlation metrics are natural and crucial when relative trends matter more than absolute scale, \eg, in transcriptomics or human-annotated regression tasks where annotator-dependent scale shifts make absolute errors less informative \citep{skinnider2019evaluating}. However, they depend on global means, variances, and covariances. Replacing these population moments with mini-batch estimates makes the objective batch-dependent, so the resulting gradients are noisy and biased, especially in the small-batch regime.

Existing methods are incomplete. Larger batches reduce moment noise but are often memory-intensive and may hurt generalization \citep{goyal2017train,keskar2017largebatch}. Beyond large-batch strategy, prior work on non-decomposable objectives includes primal--dual or constrained formulations \citep{sanyal2018optimizing}, general stochastic compositional optimization \citep{wang2017stochastic}, coupled/compositional methods for AP and AUPRC-style objectives \citep{wang2022finite,qi2021soap}, and DRO-based formulations for partial AUC \citep{zhu2022auc}. These advances are important, but they mostly target classification metrics and do not directly resolve mini-batch correlation optimization in regression within a standard SGD pipeline.

We view this mismatch as a \emph{batch-to-population} (B2P) gap: the target correlation objective is global, but training only observes noisy local statistics. Our key observation is that, although PCC/CCC are non-decomposable as scalar objectives, their population gradients admit decomposable per-example forms whose coefficients depend only on a few global moments.

\paragraph{Contributions.}
We propose \textbf{B2P-Corr} (\Cref{fig:teaser}), which bridges this B2P gap by decoupling moment estimation from backpropagation via stop-gradient pseudo-losses and lightweight moment sketches. Specifically, we

(i) derive population gradients of PCC and CCC as decomposable per-example expectations depending on global prediction moments, and implement them via pseudo-losses with stop-gradient sketches.

(ii) introduce \textbf{B2P-EMA}, which tracks global moments with EMA, and under a standard fast-sketch / slow-parameter regime, it yields plug-in gradients that asymptotically track the moving population target at $O(B)$ cost;

(iii) further propose \textbf{B2P-U}, which is implemented via sampled incomplete U-statistics, with a reservoir-based stabilization to further stabilize moment estimation while remaining $O(B)$ cost.

Theoretically, we quantify the $O(B^{-1})$ bias of naive mini-batch PCC/CCC gradients, establish a tracking theorem for B2P-EMA, and provide an idealized control-variate analysis for B2P-U. We empirically show on synthetic and real-world benchmarks that B2P-Corr stabilizes training and improves correlation with negligible overhead.

% \paragraph{Contributions.}
% We propose \textbf{B2P-Corr}, a novel framework for optimizing PCC/CCC under mini-batch training:

% 1. We formalize the population gradients of PCC and CCC as expectations over per-example weights depending on global prediction moments, and show how to realize them via decomposable pseudo-losses and stop-gradient operators.

% 2. We introduce variant \textbf{B2P-EMA}, which maintains exponentially weighted running estimates of global moments and, under a standard fast-sketch / slow-parameter regime, yields plug-in gradients consistent with the moving-target population gradient with negligible cost.

% 3. We further propose \textbf{B2P-U}, combining incomplete U-statistics \citep{clemenccon2016incomplete} with cross-batch pairing and a control-variate to further reduce estimation variance.

% Theoretically, we quantify the $O(1/B)$ population-gradient bias of naive mini-batch PCC/CCC optimization, establish a moving-target tracking theorem for B2P-EMA, and give variance-reduction guarantees for B2P-U. Empirically, we validate on controlled synthetic diagnostics and three real-world benchmarks, showing B2P-Corr is a practical drop-in replacement that stabilizes training and improves correlation performance.

\section{Related Work}
\label{sec:related_work}

% \paragraph{Correlation Objectives in Deep Learning}
% Correlation-based losses are widely used when evaluation emphasizes agreement in shape or calibration. In affective computing and emotion recognition, CCC is a standard metric and is often optimized directly as a loss \citep{atmaja2020evaluation,praveen2024recursive,zhang2022continuous}. Prior work consistently observes improved CCC compared to MSE objective, but also reports optimization instability and the frequent need for MSE warmup or careful learning rate schedules \citep{atmaja2020evaluation}, highlighting the difficulty of mini-batch correlation training.

% In multi-view representation learning, deep Canonical Correlation Analysis (CCA) maximizes cross-view correlation, but relies on accurate covariance estimates and is known to require large batches or additional regularization to avoid ill-conditioned covariance estimates \citep{andrew2013deepcca,sdgcca2023}. Since large batches can be infeasible and may degrade generalization \citep{goyal2017train,keskar2017largebatch}, stabilizing correlation optimization under small-batch setting remains an important practical problem.

\paragraph{Correlation Objectives in Deep Learning}
Correlation-based losses are common when evaluation emphasizes agreement in shape or calibration. In affective computing and emotion recognition, CCC is a standard metric and is often optimized directly, typically improving over MSE but also requiring MSE warmup or careful schedules for stability \citep{atmaja2020evaluation,praveen2024recursive,zhang2022continuous}. In multi-view representation learning, deep CCA similarly depends on reliable covariance estimates and often requires large batches or additional regularization \citep{andrew2013deepcca,sdgcca2023}. Since large batches can be impractical and may hurt generalization \citep{goyal2017train,keskar2017largebatch}, stable small-batch correlation optimization remains an important practical problem.

\paragraph{Optimization of Non-decomposable Objectives}
Most work on non-decomposable stochastic optimization focuses on classification metrics such as F-measure, AUC, and AUPRC. Beyond primal--dual and saddle-point methods for confusion-matrix ratios \citep{sanyal2018optimizing,puthiya2014optimizing,pmlr-v28-dembczynski13,gao2013one,yuan2021large}, recent work studies compositional/coupled formulations and other gradient-based surrogates, including SCGD for compositions of expectations \citep{wang2017stochastic}, finite-sum coupled compositional optimization for AP and related objectives \citep{wang2022finite}, AUPRC optimization via dependent compositional objectives \citep{qi2021soap}, and pAUC maximization via DRO-based methods \citep{zhu2022auc}. Incomplete U-statistics provide another $O(B)$ route for pairwise objectives \citep{clemenccon2016incomplete,sgdIncompleteU2015}. Our method is related in spirit but targets PCC/CCC regression, using global moment sketches and stop-gradient pseudo-losses to remain within a standard SGD pipeline.

% \paragraph{Optimization of Non-decomposable Objectives}

% Most non-decomposable optimization work targets classification metrics such as F-measure and AUC. Primal--dual and saddle-point methods optimize ratios of confusion-matrix terms \citep{sanyal2018optimizing}, with specialized approaches for F-measures \citep{puthiya2014optimizing,pmlr-v28-dembczynski13} and AUC \citep{gao2013one,yuan2021large}. While effective in some settings, these methods typically introduce auxiliary variables and can exhibit min--max instability. Incomplete U-statistics is another line of work that provides $O(B)$ stochastic optimization for pairwise objectives \citep{clemenccon2016incomplete,sgdIncompleteU2015}. Our method is inspired by these ideas but focuses on correlation metrics and remaining within a standard SGD paradigm. We adapt these tools to PCC/CCC regression by maintaining global moment sketches and using stop-gradient pseudo-losses, yielding stable $O(B)$ training within standard SGD.

\section{Correlation Objectives and B2P-Corr Decomposition}
\label{sec:prelim}

\subsection{Problem Setup and Global Moments}

We consider supervised regression with data $(x,y)\sim\mathcal{D}$, where $x\in\mathcal{X}$ and $y\in\mathbb{R}$ (straightforward to extend to vector outputs). A model $f_\theta:\mathcal{X}\to\mathbb{R}$ produces predictions $f_\theta(x)$ with parameters $\theta\in\mathbb{R}^d$.

We denote batch size with $B$ and the population moments of targets and predictions by
\begin{gather}
\mu_y = \E[y], \quad \mu_f = \E[f_\theta(x)], \quad \sigma_y^2 = \Var(y),\quad \sigma_f^2 = \Var(f_\theta(x)), \\
C = \Cov(y,f_\theta(x)) = \E\!\left[(y-\mu_y)(f_\theta(x)-\mu_f)\right].
\end{gather}
Throughout, $\sigma_y^2$ and $\mu_y$ are independent of $\theta$, while $(\mu_f,\sigma_f^2,C)$ depend on $\theta$ through $f_\theta$.

The population PCC and CCC between $f_\theta(x)$ and $y$ can be expressed as
\begin{equation}
\rho(\theta) =\frac{\Cov(y,f_\theta(x))}{\sigma_y \sigma_f} =\frac{C}{\sigma_y \sigma_f}, \,\,\, \mathrm{CCC}(\theta) = \frac{2C}{\sigma_y^2 + \sigma_f^2 + (\mu_y - \mu_f)^2} = \frac{2C}{D},
\end{equation}

where $D:=\sigma_y^2+\sigma_f^2+(\mu_y-\mu_f)^2$. Both objectives are \emph{non-decomposable}: they depend on global moments $(\mu_f,\sigma_f^2,C)$, which are unknown and must be estimated from data.

\paragraph{Stop-gradient Operator.}
We use $\operatorname{sg}(\cdot)$ to denote a stop-gradient operator: the value of $\operatorname{sg}(u)$ is equal to $u$ in the forward pass, but $\partial\operatorname{sg}(u)/\partial u\! =\! 0$ in backprop. In autodiff frameworks (\eg, PyTorch \citep{paszke2019pytorch}), this corresponds to \texttt{detach} or \texttt{stop\_gradient}. This will let us treat moment sketches as constants when differentiating pseudo-losses.

\subsection{Population Gradient Decompositions}
\label{sec:grad-identities}

Although PCC/CCC are non-decomposable as scalar objectives, we show their \emph{population gradients} admit a clean per-example decomposition where only a few global moments appear as coefficients. These identities are the basis of the B2P-Corr framework.

\begin{lemma}[Population gradient of PCC]
\label{lem:pcc-grad}
Let $\rho(\theta) = C/(\sigma_y \sigma_f)$ with moments defined above. Then
\begin{equation*}
\label{eq:pcc-grad}
\nabla_\theta \rho
\! =\! \E\!\Bigg[
\Bigg(
\frac{y-\mu_y}{\sigma_y \sigma_f}
-
\rho\,\frac{f_\theta(x)-\mu_f}{\sigma_f^2}
\Bigg)
\nabla_\theta f_\theta(x)
\Bigg].
\end{equation*}
\end{lemma}

\begin{lemma}[Population gradient of CCC]
\label{lem:ccc-grad}
Recall $\mathrm{CCC}(\theta)=2C/D$ with $C$ and $D$ defined above, then
\begin{equation*}
\label{eq:ccc-grad}
\nabla_\theta \mathrm{CCC}\!
=\!
\E\Bigg[\!
\frac{2}{D^2}\!\Big(
D (y-\mu_y) - 2C (f_\theta(x)-\mu_y)
\Big)\nabla_\theta f_\theta(x)\!
\Bigg].
\end{equation*}
\end{lemma}

Proofs are shown in \Cref{sec:appendix-proof-pcc,sec:appendix-proof-ccc}. Lemma~\ref{lem:pcc-grad} and~\ref{lem:ccc-grad} show that, if we can provide accurate estimates of the global moments,
then correlation learning reduces to optimizing an expectation of \emph{per-example} weighted gradients.

\subsection{B2P Decomposition via Stop-gradient Pseudo-losses}
\label{sec:pseudo-loss}

For scalar predictions $f=f_\theta(x)$ and target $y$, consider the pseudo-losses
\begin{align}
\mathcal{L}_\rho(f;y)\!
=\! -&\frac{y-\operatorname{sg}(\mu_y)}{\operatorname{sg}(\sigma_y)\operatorname{sg}(\sigma_f)} f\!
+\! \frac{\operatorname{sg}(\rho)}{2\operatorname{sg}(\sigma_f)^2}\big(f\!-\!\operatorname{sg}(\mu_f)\big)^2,
\label{eq:pseudo-pcc}\\[4pt]
\mathcal{L}_{\mathrm{CCC}}(f;y)
=& -\frac{2}{\operatorname{sg}(D)^2}\Big(
\operatorname{sg}(D)(y-\operatorname{sg}(\mu_y))f
-\operatorname{sg}(C)\big(f-\operatorname{sg}(\mu_y)\big)^2
\Big).
\label{eq:pseudo-ccc}
\end{align}
Note that these pseudo-losses need not equal $-\rho$ or $-\mathrm{CCC}$, we only require its gradient to match the target population gradient when the moments are correct. Treating the moments inside $\operatorname{sg}(\cdot)$ as constants, the derivative w.r.t.\ $f$ gives
\begin{equation}
\frac{\partial \mathcal{L}_\rho}{\partial f}
= -\frac{y-\mu_y}{\sigma_y \sigma_f} + \rho\,\frac{f-\mu_f}{\sigma_f^2}, \quad \frac{\partial \mathcal{L}_{\mathrm{CCC}}}{\partial f}
= -\frac{2}{D^2}\Big(D(y-\mu_y)-2C(f-\mu_y)\Big).
\end{equation}
Therefore backpropagation yields $\E[\nabla_\theta \mathcal{L}_\rho]\!\!=\!\!-\nabla_\theta\rho$ (negative because we maximize PCC/CCC) and
$\E[\nabla_\theta \mathcal{L}_{\mathrm{CCC}}]\!=\!-\nabla_\theta\mathrm{CCC}$ provided we plug in the true moments. In standard offline setting, $(\mu_y,\sigma_y^2)$ can be precomputed from training labels. In online settings, we simply maintain a EMA sketch as $(\hat\mu_y,\widehat{\sigma_y^2})$ since these moments are $\theta$-independent and will not alter the gradient decompositions.

Now, pseudo-losses in \Cref{eq:pseudo-pcc,eq:pseudo-ccc} are decomposable given the moments, but in practice we do not know the true moments. In next section, we introduce our B2P-Corr which maintains lightweight ``moment sketches" with two $O(B)$ sketching strategies, B2P-EMA and B2P-U.

\begin{figure}[t]
    \centering
    \includegraphics[width=1\textwidth]{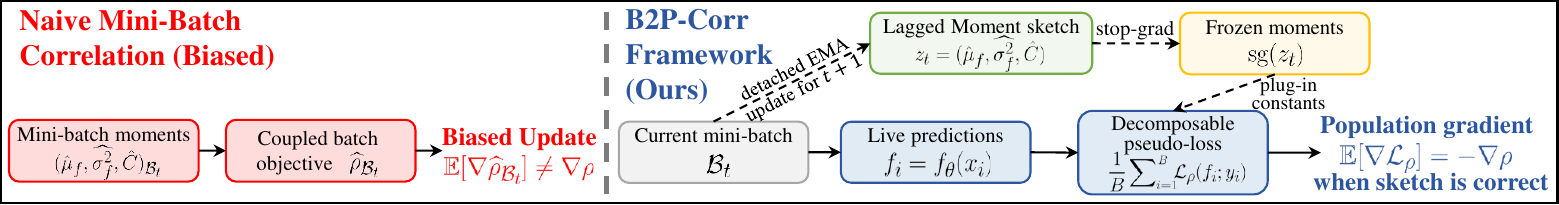}
    \caption{Batch-to-population mismatch in naive mini-batch and our proposed B2P-Corr fix.}
    \label{fig:teaser}
    \vspace{-12pt}
\end{figure}

\section{Method: B2P-Corr}
\label{sec:method}

We propose B2P-Corr, a simple drop-in framework for mini-batch optimization of PCC/CCC which
(i) use the decomposable pseudo-losses in \Cref{eq:pseudo-pcc,eq:pseudo-ccc}, and
(ii) maintain a small ``moment sketch" of the required global moments, treated as constants using $\operatorname{sg}(\cdot)$.
The remaining design choice is how to update the sketch from mini-batches.
We provide two practical instantiations:
\textbf{B2P-EMA} a running-moment sketch and \textbf{B2P-U} a variance-reduced replacement
for the covariance/variance estimates used inside B2P-EMA. Both share the same pseudo-loss.

\subsection{B2P-EMA: Exponential-moving Moment Sketch}
\label{sec:b2p-ema}

B2P-EMA maintains exponentially weighted average (EMA) estimates of the prediction moments ($\hat\mu_f$, $\widehat{\sigma_f^2}$, $\hat C$). On each batch
$\mathcal{B}=\{(x_i,y_i)\}_{i=1}^B$, compute predictions $f_i=f_\theta(x_i)$ and define
$\tilde f_i=\operatorname{sg}(f_i)$ for moment updates. At iteration $t$, B2P-EMA computes
\begin{gather}
\overline f_{\mathcal{B}} \!=\! \tfrac{1}{B}\sum_i \tilde f_i, \,\,\, \mathrm{var}_{\mathcal{B}}(f) \!=\! \tfrac{1}{B-1}\sum_i(\tilde f_i-\overline f_{\mathcal{B}})^2, \,\,\, \mathrm{cov}_{\mathcal{B}}(y,f) \!=\! \tfrac{1}{B-1}\sum_i (y_i-\mu_y)(\tilde f_i-\overline f_{\mathcal{B}}),
\end{gather}
and updates
\begin{gather}
\hat\mu_f \!\leftarrow\! (1-\alpha_t)\hat\mu_f\!+\!\alpha_t\overline f_{\mathcal{B}}, \,\,\,
\widehat{\sigma_f^2} \!\leftarrow\! (1\!-\!\alpha_t)\widehat{\sigma_f^2}\!+\!\alpha_t\mathrm{var}_{\mathcal{B}}(f),\,\,\,
\hat C \!\leftarrow\! (1-\alpha_t)\hat C+\alpha_t\,\mathrm{cov}_{\mathcal{B}}(y,f), \label{eq:ema-updates}
\end{gather}

then form derived quantities (with $\varepsilon,\delta>0$ for stability)
\begin{gather}
\hat\sigma_f=\sqrt{\widehat{\sigma_f^2}+\varepsilon},\,\,\,\,
\hat\rho=\mathrm{clip}\!\left(\frac{\hat C}{\sigma_y\hat\sigma_f},-1+\delta,1-\delta\right)\!,\,\,\,
\hat D=\sigma_y^2+\widehat{\sigma_f^2}+(\mu_y-\hat\mu_f)^2.
\end{gather}
Finally, optimize the decomposable batch objective $\frac{1}{B}\sum_i \mathcal{L}_\rho(f_i;y_i)$
or $\frac{1}{B}\sum_i \mathcal{L}_{\mathrm{CCC}}(f_i;y_i)$ where all estimated moments
appear under $\operatorname{sg}(\cdot)$. \Cref{alg:b2p-ema} summarizes one step and the extra cost over naive mini-batch training is $O(B)$ update with negligible overhead.

\noindent\textbf{Remark.} Theorem~\ref{prop:ema-consistency} uses a diminishing $\alpha_t$ for a clean asymptotic tracking statement. In experiments we use a constant $\alpha$, which corresponds to keeping a small steady-state tracking error when the target moments drift slowly.

\subsection{B2P-U: Stabilized Plug-in Sketch}
\label{sec:b2p-u}

While B2P-EMA greatly stabilizes correlation optimization, $\mathrm{var}_{\mathcal{B}}(f)$ and $\mathrm{cov}_{\mathcal{B}}(y,f)$ still fluctuate under small $B$. B2P-U further stabilizes these moment estimates using cross-batch U-statistic identities together with a small reservoir-based correction.

\paragraph{Cross-batch U-statistic identities.}
For i.i.d.\ pairs $((x_1,y_1),(x_2,y_2))$ and predictions $f_i=f_\theta(x_i)$, the moments satisfy
\begin{gather}
\Cov(y,f) = \tfrac12\,\E\big[(y_1-y_2)(f_1-f_2)\big], \,\,\,\,\,\,
\Var(f)   = \tfrac12\,\E\big[(f_1-f_2)^2\big]. \label{eq:u-var}
\end{gather}
with proof in \Cref{sec:appendix-ustat}. For the second-order moments considered here, the complete cross-batch estimator over all $B^2$ pairs also admits an exact algebraically collapsed $O(B)$ form. In our implementation we instead use a sampled incomplete U-statistic with $K=O(B)$ cross-batch pairs, which preserves the same pairwise identity while fitting naturally with our cross-batch sketch update. To reduce coupling between the sketch (in denominators) and the current per-example gradients, we form pairs across consecutive batches: at step $t$, sample cross-batch pairs $(y_i,\tilde f_i)$ and $(y'_j,\tilde f'_j)$ with $i\!\in\!\mathcal{B}_t$, $j\!\in\!\mathcal{B}_{t-1}$ and compute
\begin{gather}
\hat C_{\mathrm{U}}=\frac{1}{2K}\sum_{(i,j)} (y_i-y'_j)(\tilde f_i-\tilde f'_j),
\,\,\,\,\,\,\,
\widehat{\sigma^2_{f,\mathrm{U}}}=\frac{1}{2K}\sum_{(i,j)} (\tilde f_i-\tilde f'_j)^2.
\end{gather}

\paragraph{Reservoir-based stabilization.}
We maintain a small reservoir $\mathcal{R}$ (size $m\ll|\mathcal{D}|$) of cached $\tilde f$ and labels $y$ (refreshed periodically). For any moment $M\in\{C,\sigma_f^2\}$, we combine the sampled U-statistic estimate $\hat M_{\mathrm{U}}$ with a reservoir estimate $\hat M_{\mathcal{R}}$ via
\[
\hat M^\star=\hat M_{\mathrm{U}}+\lambda\big(\hat M_{\mathcal{R}}-\widehat{\E}[\hat M_{\mathcal{R}}]\big),
\]
where $\widehat{\E}[\hat M_{\mathcal{R}}]$ is a running mean of $\hat M_{\mathcal{R}}$ tracked by an EMA. In the classical idealized control-variate setting, the optimal choice is
\begin{equation}
\label{eq:lambda-opt}
\lambda^\star = -\frac{\Cov(\hat M_{\mathrm{U}},\hat M_{\mathcal{R}})}{\Var(\hat M_{\mathcal{R}})}.
\end{equation}
In our implementation, however, the reservoir is disjoint and periodically refreshed, so the EMA-estimated coefficient should be viewed as a practical stabilization heuristic rather than a guarantee of stepwise conditional variance reduction at fixed $\theta_t$; it can also reflect temporal correlation induced by slow parameter drift. We then update $(\hat C, \widehat{\sigma_f^2})$ in \Cref{eq:ema-updates} using $\hat C^\star$ and $\widehat{\sigma_f^2}{}^\star$ in place of the raw mini-batch estimators. \Cref{alg:b2p-u} details the B2P-U subroutine, which preserves $O(B)$ complexity in our implementation.

\begin{algorithm}[t]
\caption{B2P-EMA training step for PCC/CCC}
\label{alg:b2p-ema}
\begin{algorithmic}[1]
\State \textbf{Input:} mini-batch $\mathcal{B}=\{(x_i,y_i)\}_{i=1}^{B\geq2}$; moment sketch ($\hat\mu_f,\widehat{\sigma_f^2},\hat C$); label moments $(\mu_y,\sigma_y)$; stability $\varepsilon,\delta>0$; EMA rate $\alpha_t\in(0,1]$
\State $f_i \gets f_\theta(x_i)$ for all $i\in\mathcal{B}$; \; $\tilde f_i \gets \operatorname{sg}(f_i)$
\State $\overline{f}_\mathcal{B} \!\gets\! \frac{1}{B}\sum_{i=1}^B \tilde f_i$; $\mathrm{var}_\mathcal{B}(f) \!\gets\! \frac{1}{B-1}\sum_{i=1}^B (\tilde f_i-\overline f_{\mathcal B})^2$;
$\mathrm{cov}_\mathcal{B}(y,f) \!\gets\! \frac{1}{B-1}\sum_{i=1}^B (y_i-\mu_y)(\tilde f_i-\overline{f}_\mathcal{B})$
\State $\hat\mu_f \!\gets\! (1-\alpha_t)\hat\mu_f + \alpha_t\overline{f}_\mathcal{B}$; $\widehat{\sigma_f^2} \!\gets\! (1-\alpha_t)\widehat{\sigma_f^2} + \alpha_t \mathrm{var}_\mathcal{B}(f)$; $\hat C \gets (1-\alpha_t)\hat C + \alpha_t\,\mathrm{cov}_\mathcal{B}(y,f)$
\State $\hat\sigma_f \gets \sqrt{\widehat{\sigma_f^2}+\varepsilon}$; \; $\hat\rho \gets \mathrm{clip}\big(\hat C/(\sigma_y\hat\sigma_f), -1+\delta, 1-\delta\big)$; $\hat D \gets \sigma_y^2+\widehat{\sigma_f^2}+(\mu_y-\hat\mu_f)^2$ \Comment{CCC}
\State $\mathcal{L}\gets \frac{1}{B}\sum_{i=1}^B \mathcal{L}(f_i;y_i)$ using \Cref{eq:pseudo-pcc} (use \Cref{eq:pseudo-ccc} if $\text{loss = CCC}$)
\State Compute $\nabla_\theta \mathcal{L}$ and update $\theta$ with an optimizer step
\end{algorithmic}
\end{algorithm}

\section{Theoretical Study}
\label{sec:theory}

We analyze three aspects of correlation optimization:
(i) the bias of naive mini-batch PCC/CCC gradients;
(ii) moving-target tracking of the B2P-EMA sketch and the induced plug-in gradient; and
(iii) an idealized control-variate analysis that motivates the reservoir stabilization used in B2P-U.

% \subsection{Standing Assumptions}
% \label{sec:theory-assumptions}

\begin{assumption}[Differentiability, Integrability, and Non-degeneracy]
\label{assump:standing}
Fix $\theta$ and write $f\!:=\!f_\theta(x)$. Define the raw-moment vector
$s\!:=\! (f, f^2, y f) \in \mathbb{R}^3$, and its Jacobian $\nabla_\theta s \in \mathbb{R}^{3\times d}$. Let $\|\cdot\|$ denote the Euclidean norm on $\mathbb{R}^3$ and $\|\cdot\|_F$ the Frobenius norm on $\mathbb{R}^{3\times d}$. Assume:

\begin{enumerate}[label=(\arabic*), leftmargin=1.5em]
\item \textbf{Interchanging $\nabla_\theta$ and $\E$.}
$f_\theta(x)$ is a.s.\ differentiable in $\theta$ and for each coordinate of $s$ we may interchange differentiation and expectation, \eg,
$\nabla_\theta \E[f]=\E[\nabla_\theta f]$, $\nabla_\theta \E[f^2]=\E[\nabla_\theta(f^2)]$, and $\nabla_\theta \E[yf]=\E[y\nabla_\theta f]$.

\item \textbf{Finite moments.}
  $\E[y^2]<\infty$,
  $\E[\|s\|^2] < \infty$,
  $\E[\|\nabla_\theta s\|_F^2] < \infty$,
  and 
  $\E[\|s\|^2\,\|\nabla_\theta s\|_F] < \infty$.

\item \textbf{Non-degeneracy.}
$\sigma_y\!>\!0$ and either
(i) $\sigma_f(\theta)\!\ge\! \sigma_{\min}\!>\!0$, or
(ii) we analyze $\varepsilon$-stabilized versions with $\sigma_f^2\leftarrow \sigma_f^2+\varepsilon$ and $D\leftarrow D+\varepsilon$.
\end{enumerate}
\end{assumption}
\setlength{\textfloatsep}{5pt plus 1pt minus 1pt}
\begin{algorithm}[t]
\caption{B2P-U sketch update plug-in for \Cref{alg:b2p-ema}}
\label{alg:b2p-u}
\begin{algorithmic}[1]
\State \textbf{Input:} current batch $(y_i,\tilde f_i)_{i=1}^B$ and previous batch $(y'_j,\tilde f'_j)_{j=1}^B$;
reservoir $\mathcal{R}$ (size $m$) storing cached pairs $(y,\tilde f)$; current sketch $(\hat C,\widehat{\sigma_f^2})$; EMA rate $\alpha_t$; number of pairs $K=O(B)$; $\tau>0$
\State Sample $K$ cross-batch pairs $(i,j)$ with $i,j\!\in\!\!\{1,\dots,B\}$
\State $\hat C_{\mathrm{U}} \gets \frac{1}{2K}\sum_{(i,j)} (y_i-y'_j)\,(\tilde f_i-\tilde f'_j)$;
$\,\,$ $\widehat{\sigma^2_{f,\mathrm{U}}} \gets \frac{1}{2K}\sum_{(i,j)} (\tilde f_i-\tilde f'_j)^2$
\State Update running means $\widehat{\E}[\hat C_{\mathcal{R}}]$, $\widehat{\E}[\widehat{\sigma^2_{f,\mathcal{R}}}]$ (e.g., via EMA)
\State Update running estimates of $\Var(\hat C_{\mathcal{R}})$, $\Cov(\hat C_{\mathrm{U}},\hat C_{\mathcal{R}})$ and
$\Var(\widehat{\sigma^2_{f,\mathcal{R}}})$, $\Cov(\widehat{\sigma^2_{f,\mathrm{U}}},\widehat{\sigma^2_{f,\mathcal{R}}})$
\State $\lambda_C \gets -\Cov(\hat C_{\mathrm{U}},\hat C_{\mathcal{R}})/\big(\Var(\hat C_{\mathcal{R}})+\tau\big)$; $\,\,$ $\lambda_V \gets -\Cov(\widehat{\sigma^2_{f,\mathrm{U}}},\widehat{\sigma^2_{f,\mathcal{R}}})/\big(\Var(\widehat{\sigma^2_{f,\mathcal{R}}})+\tau\big)$
\State $\hat C^\star \gets \hat C_{\mathrm{U}} + \lambda_C\big(\hat C_{\mathcal{R}}-\widehat{\E}[\hat C_{\mathcal{R}}]\big)$; $\,\,$ $\widehat{\sigma_f^2}{}^\star \gets \widehat{\sigma^2_{f,\mathrm{U}}} + \lambda_V\big(\widehat{\sigma^2_{f,\mathcal{R}}}-\widehat{\E}[\widehat{\sigma^2_{f,\mathcal{R}}}]\big)$
\State $\hat C \gets (1-\alpha_t)\hat C + \alpha_t\,\hat C^\star$; 
$\,\,$ $\widehat{\sigma_f^2} \gets (1-\alpha_t)\widehat{\sigma_f^2} + \alpha_t\,\widehat{\sigma_f^2}{}^\star$ \Comment{Reduced-variance sketch}
\State Proceed with lines 5--7 in \Cref{alg:b2p-ema}
\end{algorithmic}
\end{algorithm}

\subsection{Bias of Naive Mini-batch PCC/CCC Gradients}
\label{sec:theory-bias}

Let $\widehat{\rho}_{\mathcal{B}}$ and $\widehat{\mathrm{CCC}}_{\mathcal{B}}$ denote PCC and CCC computed from a random i.i.d.\ mini-batch
$\mathcal{B}$ of size $B\ge 2$, where the moments inside the ratio are computed from the same batch and we differentiate through those batch moments. This is the standard ``naive'' mini-batch training strategy.

\begin{proposition}[Naive Mini-batch Gradients Have $O(1/B)$ Bias]
\label{prop:bias}
Under Assumption~\ref{assump:standing}, for i.i.d.\ mini-batches of size $B$ we have
\begin{align}
\E_{\mathcal{B}}\big[\nabla_\theta \widehat{\rho}_{\mathcal{B}}\big]
= \nabla_\theta \rho + O(B^{-1}),\,\,\,\,\,
\E_{\mathcal{B}}\big[\nabla_\theta \widehat{\mathrm{CCC}}_{\mathcal{B}}\big]
= \nabla_\theta \mathrm{CCC} + O(B^{-1}),
\end{align}
where the $O(B^{-1})$ term is generally nonzero due to coupling between random batch moments
(appearing in denominators) and per-example gradient terms.
A full proof is given in Appendix~\ref{sec:appendix-proof-bias}. An empirical verification with synthetic dataset is given in~\Cref{sssec:synth_results}.
\end{proposition}
This bias is the analytic expression of the B2P gap: naive mini-batch optimization directly follows the gradient of a \emph{random} local objective whose expectation does not equal the desired population objective, and the discrepancy decreases only slowly as $B$ increases.

\subsection{Tracking of B2P-EMA}
\label{sec:theory-b2p}

In practice, the parameters and the sketch are updated jointly, so the key question is whether the EMA sketch tracks the \emph{current} population moments along training. Let $m(\theta):=(\mu_f(\theta),\sigma_f^2(\theta),C(\theta))$, $z_t:=(\hat\mu_f^{(t)},\widehat{\sigma_f^2}{}^{(t)},\hat C^{(t)})$, and let $\mathcal F_t$ denote all information available before drawing the batch at step $t$. This is the standard fast-sketch / slow-parameter regime in stochastic compositional optimization \citep{wang2017stochastic,wang2022finite}. We work under a standard tracking regime: the detached batch-moment estimator is conditionally unbiased with bounded second moment, the EMA steps satisfy standard diminishing-step conditions, the target moments drift slowly relative to $\alpha_t$, and the mean plug-in gradient is Lipschitz in the sketch. The formal assumptions are deferred to Appendix~\ref{app:ema-tracking-assumptions}.

\begin{theorem}[B2P-EMA tracks moving population moments]
\label{prop:ema-consistency}
Let $J(\theta)$ denote the negative PCC objective $-\rho(\theta)$ or the negative CCC objective $-\mathrm{CCC}(\theta)$, and let $g_t$ be the mini-batch gradient produced by the B2P-EMA pseudo-loss at step $t$. Under Assumption~\ref{assump:standing} and~\ref{assump:ema},
\begin{gather*}
\E\|z_t-m(\theta_t)\|^2 \to 0, \,\,\,\,
\E\big\|\E[g_t\mid \mathcal F_t]-\nabla J(\theta_t)\big\| \to 0.
\end{gather*}
\end{theorem}

A sufficient condition for slow drift is that $m(\theta)$ is Lipschitz and $\|\theta_{t+1}-\theta_t\|=o(\alpha_t)$. Proof is given in Appendix~\ref{sec:appendix-proof-ema}. This is a tracking result for the sketch and induced plug-in gradient, not a full convergence theorem for the coupled nonconvex training dynamics.

\subsection{Idealized Control-Variate Analysis for B2P-U}
\label{sec:theory-var}

\begin{proposition}[Idealized control-variate analysis]
\label{prop:var-red}
Let $\hat M_{\text{U}}$ be an (incomplete) U-statistic estimate of a scalar moment $M$,
and let $\hat M_{\mathcal{R}}$ be another estimator with $\E[\hat M_{\mathcal{R}}]=M$ and $\Var(\hat M_{\mathcal{R}})>0$.
For any scalar $\lambda$, define
$
\hat M^\star = \hat M_{\text{U}} + \lambda\big(\hat M_{\mathcal{R}}-\E[\hat M_{\mathcal{R}}]\big).
$
Then $\hat M^\star$ is unbiased for $M$, and choosing
$\lambda^\star = -\Cov(\hat M_{\text{U}},\hat M_{\mathcal{R}})/\Var(\hat M_{\mathcal{R}})$ yields
\[
\Var(\hat M^\star)
= \Var(\hat M_{\text{U}}) - \frac{\Cov(\hat M_{\text{U}},\hat M_{\mathcal{R}})^2}{\Var(\hat M_{\mathcal{R}})}
\le \Var(\hat M_{\text{U}}),
\]
with strict inequality whenever $\Cov(\hat M_{\text{U}},\hat M_{\mathcal{R}})\neq 0$ (proof is provided in~\Cref{sec:appendix-proof-varred}).
\end{proposition}

\begin{remark}[Implementation]
Proposition~\ref{prop:var-red} is an idealized statement: it applies when the auxiliary estimator is correlated with $\hat M_{\text{U}}$ at the target parameter. In Algorithm~\ref{alg:b2p-u}, the reservoir is disjoint and periodically refreshed, so the classical conditional sampling covariance at fixed $\theta_t$ can be small and, under independent-sample idealizations, may vanish. Accordingly, the EMA-estimated coefficient in B2P-U should be interpreted as a practical stabilization mechanism---potentially reflecting temporal correlation induced by slowly drifting parameters---rather than a guarantee of stepwise conditional variance reduction. Empirically, B2P-U nonetheless exhibits lower gradient variability in \Cref{fig:exp1}.
\end{remark}

\section{Experiments}
\label{sec:experiment} 

% We first validate our theoretical analysis on two controlled synthetic datasets. In particular, we (i) empirically verify the $O(B^{-1})$ bias of naive mini-batch PCC/CCC gradients (Proposition~\ref{prop:bias}), and (ii) show that B2P-EMA and B2P-U produce gradients that better match the population correlation objective and enables stable learning under mini-batch setting. We then evaluate B2P-EMA/B2P-U as a drop-in replacement for naive mini-batch PCC/CCC objectives on three real-world datasets (\Cref{ssec:realworld}) and study computational overhead and sensitivity (\Cref{ssec:sensitivity_and_efficiency}).

We first use two synthetic experiments to verify the $O(B^{-1})$ bias of naive mini-batch PCC/CCC gradients and to compare gradient alignment/stability. We then evaluate B2P-EMA/B2P-U as drop-in replacements for naive mini-batch PCC/CCC optimization on three real-world benchmarks, followed by overhead and sensitivity studies.

\subsection{Synthetic Dataset for Theory Verification}
\label{ssec:synthetic}

We use two controlled regression problems with a linear predictor $f_\theta(x)=\theta^\top x$ to demonstrate the ``batch-to-population'' (B2P) gap for correlation objectives. Unless stated otherwise, we maximize PCC $\rho(\theta)$ defined in~\Cref{sec:prelim}. CCC diagnostics are reported in~\Cref{fig:exp1_c,fig:exp1_d} (CCC results for Exp.2 are provided in~\Cref{app:exp2_ccc}). Experiment details and metrics are deferred to~\Cref{app:experiment_settings}.

\subsubsection{Experiment Settings}
\label{sssec:synthetic_setting}
% We compare naive mini-batch correlation optimization (using batch moments) with proposed B2P-EMA and B2P-U on synthetic dataset using the following settings.

\textbf{Exp.~1: Imbalanced Mixture Regression.} We construct an imbalanced two-component Gaussian mixture to make mini-batch moment estimates highly unstable under small batch-size $B$. We sample a component indicator $s\in\{+1,-1\}$ with $\mathbb{P}(s=+1)\gg \mathbb{P}(s=-1)$, draw $x\sim\mathcal{N}(s\mu,I)$, and set
$y={w^\star}^\top x+\Delta\,s+\varepsilon_s,$
where $\Delta$ is a mixture-component label shift and $\varepsilon_s$ is zero-mean heteroscedastic noise. This setting is challenging as the population moments are well-defined, but within-batch estimates of $\Var(f_\theta)$ and $\Cov(y,f_\theta)$ vary widely across batches.

\textbf{Metric.} At each training checkpoint with $\theta$, we compute an oracle population gradient $g_{\mathrm{pop}}(\theta)=\nabla_\theta \rho(\theta)$ using full batch, and compare it with a stochastic gradient estimate $g(\theta)$ from a mini-batch of size $B$. We report (i) gradient cosine \emph{alignment} $\cos(g(\theta),g_{\mathrm{pop}}(\theta))$, and (ii) \emph{RelBias}, the relative $\ell_2$ error $\|\E[g(\theta)]-g_{\mathrm{pop}}(\theta)\|_2/\|g_{\mathrm{pop}}(\theta)\|_2.$

\textbf{Exp.~2: Isotropic Linear Regression.} We sample $x\sim\mathcal{N}(0,I)$ and let
$
y={w^\star}^\top x+\varepsilon,$
with $\varepsilon\sim\mathcal{N}(0,\sigma^2)$ and $\|w^\star\|_2=1$. Since PCC is scale-invariant in $f_\theta$, the optimum is characterized by directional alignment between $\theta$ and $w^\star$ rather than the norm of $\theta$.

\textbf{Metric.} We train with varying batch sizes and evaluate (i) \emph{DirectionRecovery} $=\cos(\theta,w^\star)$ and (ii) \emph{TestPCC} = $\rho(\theta)$ on test set.

\subsubsection{Result Analysis}
\label{sssec:synth_results}

\Cref{fig:exp1_b,fig:exp1_d} empirically verifies Proposition~\ref{prop:bias}: the average relative bias of naive PCC/CCC gradients decreases roughly linearly with $1/B$. Beyond this bias law, \Cref{fig:exp1} shows that B2P-Corr yields much better-aligned stochastic gradients, especially at very small batch sizes. For $B{=}2$, naive PCC gradients are often NaN or nearly random, whereas B2P remains stable; for $B{=}4$, alignment improves from $0.16$ to $0.43$ and RelBias drops from $0.98$ to $0.08$. B2P-U matches B2P-EMA in mean alignment while reduces gradient estimation variance (different from variance in~\Cref{tab:real_results}). Empirically, B2P-U in~\Cref{fig:exp1_a} exhibits a narrower variance shade, with average standard deviations across batch sizes dropping from $0.0346$ (naive) and $0.0275$ (B2P-EMA) to $0.0183$ (B2P-U).

\Cref{fig:exp2} shows the same optimization effect in isotropic regression. With $B{=}2$, naive mini-batch PCC training fails, whereas both B2P variants recover the correct direction and substantially higher test PCC; the gap narrows as $B$ increases, consistent with improved concentration of batch moments.

\subsection{Real-world Regression Benchmarks}
\label{ssec:realworld}

We evaluate B2P-Corr on three real-world tasks where the state-of-the-art baselines \emph{directly optimize} PCC/CCC with mini-batches. We adopt a strict drop-in protocol: for each baseline, we keep the architecture, data splits, and auxiliary losses unchanged, and only replace the naive mini-batch correlation objective with our decomposable B2P pseudo-loss (B2P-EMA or B2P-U) and a fine-grained PCC loss alternative PLC~\citep{chen2022plc}. Dataset details and training protocols are provided in \Cref{app:experiment_settings}.

\subsubsection{Datasets}
\label{sssec:realworld_dataset}

\textbf{Video Quality Assessment (PCC).}
KoNViD-1k comprises 1,200 in-the-wild videos (960$\times$540) sampled from YFCC100M~\citep{thomee2016yfcc100m}.
Following KVQ~\citep{qu2025kvq}, we use Mean Opinion Scores (MOS) as targets and split 960/120/120 videos for train/val/test.
KVQ directly optimizes PCC and reports PCC, SROCC, and RMSE.

\textbf{Spatial Transcriptomic (PCC).}
The 10xProteomic dataset used by EGN \citep{yang2023egn} contains 32,032 pathology image patches paired with gene-expression vectors $\mathbf{y}_i\in\mathbb{R}_+^{|G|}$ for breast cancer. We follow its data preprocessing and the 3-fold evaluation. EGN jointly optimizes MSE+PCC during training and reports PCC, MAE, and MSE.

\textbf{Video-based Sentiment Analysis (CCC).}
VEATIC contains 124 videos (257,601 frames) with valence $y_i^{(v)}\in\mathbb{R}$ and arousal $y_i^{(a)}\in\mathbb{R}$ annotation. The baseline optimizes joint learning objective CCC+MSE and we report CCC, PCC, and RMSE.

\begin{figure*}[t]
  \centering
  \begin{subfigure}[t]{0.25\textwidth}
    \centering
    \includegraphics[width=\linewidth]{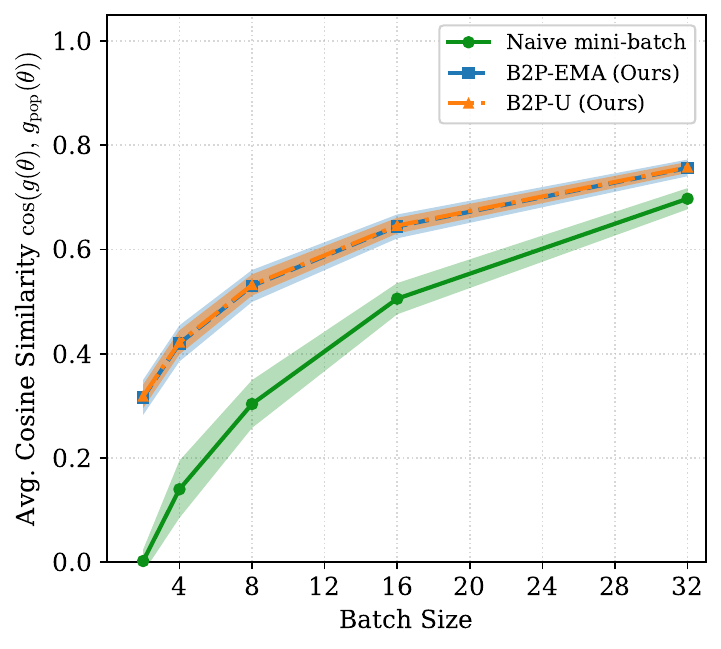}
    \caption{PCC: gradient align}
    \label{fig:exp1_a}
  \end{subfigure}\hfill
  \begin{subfigure}[t]{0.25\textwidth}
    \centering
    \includegraphics[width=\linewidth]{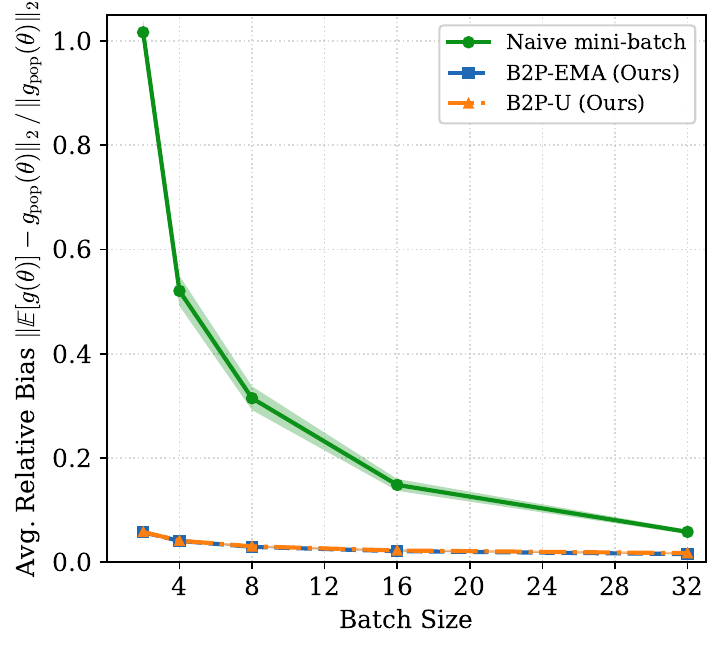}
    \caption{PCC: relative bias}
    \label{fig:exp1_b}
  \end{subfigure}\hfill
  \begin{subfigure}[t]{0.25\textwidth}
    \centering
    \includegraphics[width=\linewidth]{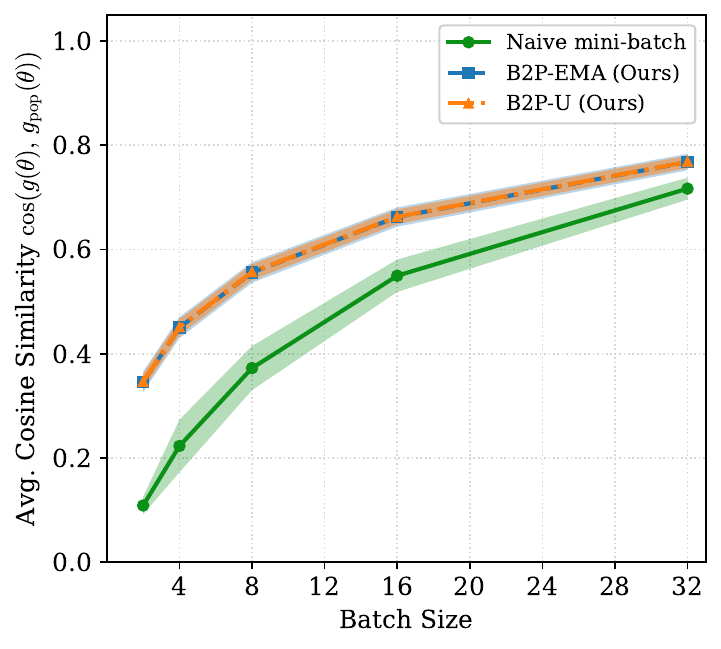}
    \caption{CCC: gradient align}
    \label{fig:exp1_c}
  \end{subfigure}\hfill
  \begin{subfigure}[t]{0.25\textwidth}
    \centering
    \includegraphics[width=\linewidth]{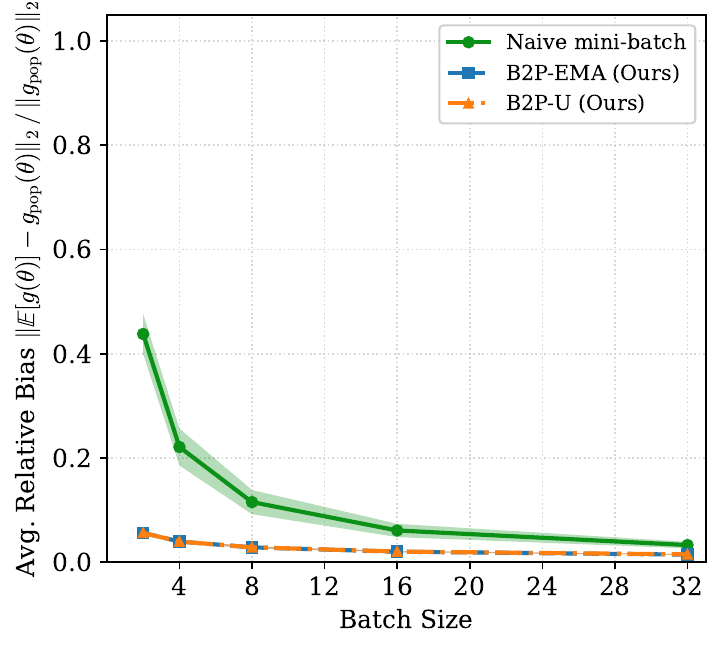}
    \caption{CCC: relative bias}
    \label{fig:exp1_d}
  \end{subfigure}

\caption{For Exp.~1, we compare stochastic gradients from naive mini-batch optimization and B2P-Corr against an oracle population gradient from full-batch training.
(a--d) Average gradient cosine alignment and relative bias at different batch size $B$ for PCC and CCC.}
\vspace{-13pt}

\label{fig:exp1}
\end{figure*}

\begin{figure*}[t]
  \centering
  \begin{subfigure}[t]{0.25\textwidth}
    \centering
    \includegraphics[width=\linewidth]{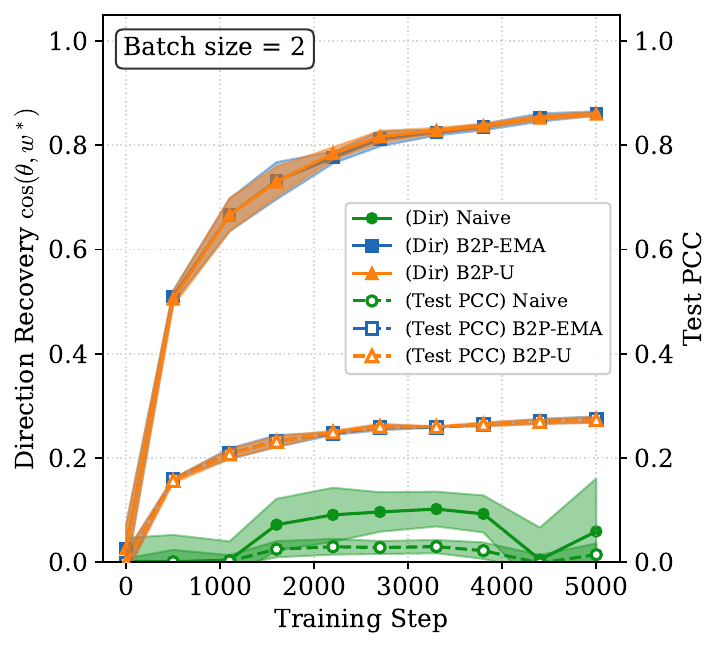}
    \caption{PCC: training ($B=2$)}
    \label{fig:exp2_a}
  \end{subfigure}\hfill
  \begin{subfigure}[t]{0.25\textwidth}
    \centering
    \includegraphics[width=\linewidth]{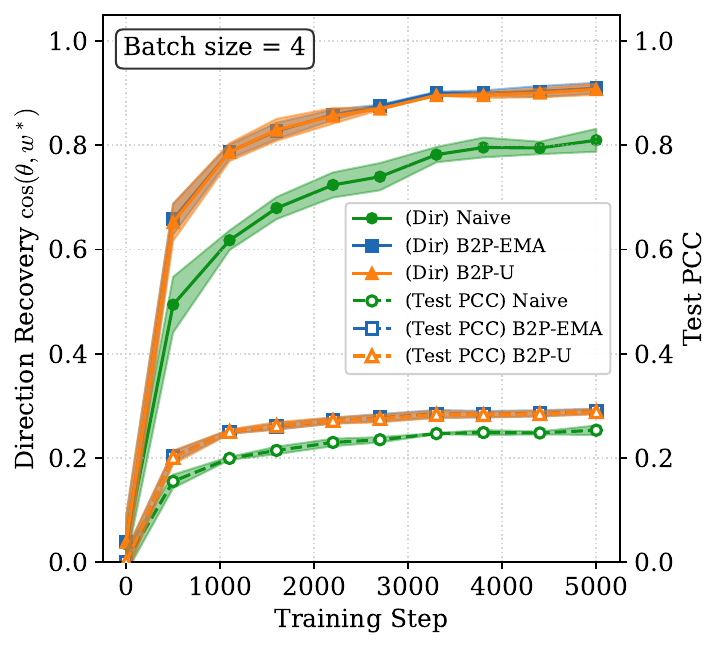}
    \caption{PCC: training ($B=4$)}
    \label{fig:exp2_b}
  \end{subfigure}\hfill
  \begin{subfigure}[t]{0.25\textwidth}
    \centering
    \includegraphics[width=\linewidth]{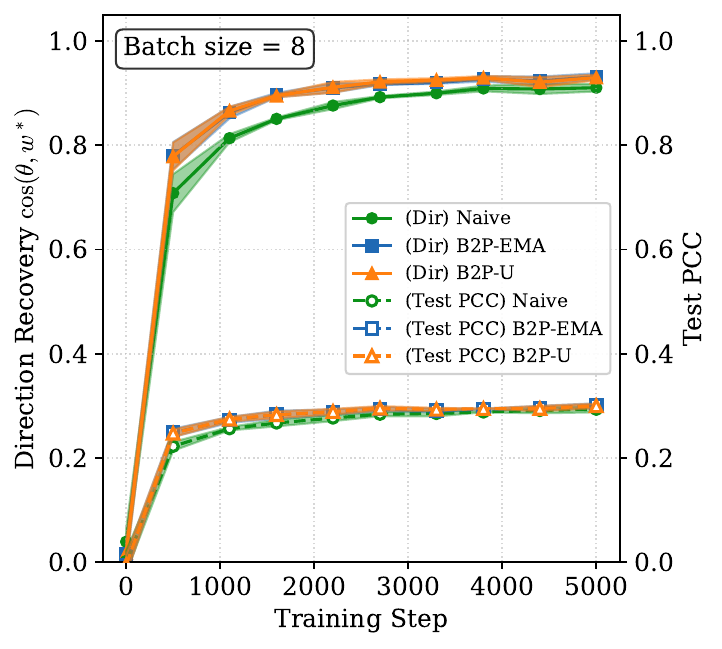}
    \caption{PCC: training ($B=8$)}
    \label{fig:exp2_c}
  \end{subfigure}\hfill
  \begin{subfigure}[t]{0.25\textwidth}
    \centering
    \includegraphics[width=\linewidth]{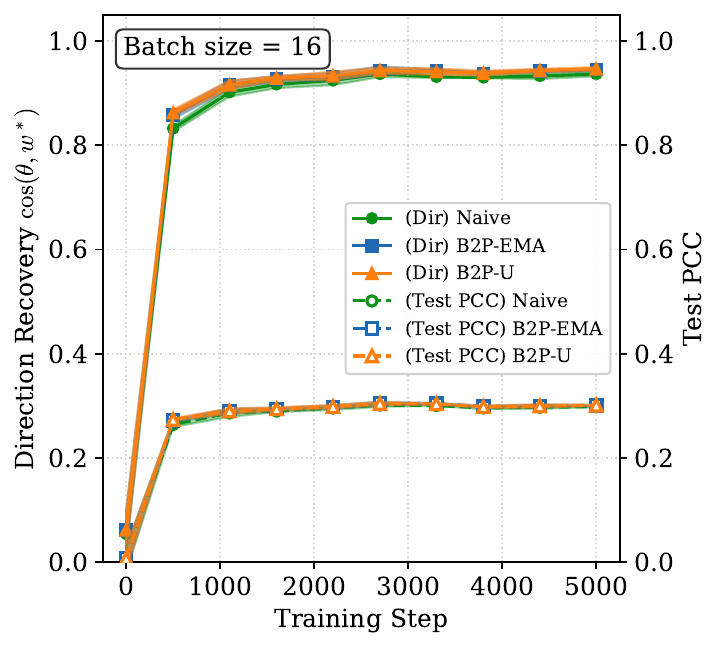}
    \caption{PCC: training ($B=16$)}
    \label{fig:exp2_d}
  \end{subfigure}
  \caption{For Exp.~2, we show DirectionRecovery and TestPCC vs. training steps for different batch sizes. With $B{=}2$, naive mini-batch fails to converge while B2P-EMA/B2P-U continue to improve. The gap narrows as $B$ increases. CCC training curves are shown in \Cref{app:exp2_ccc}.}
  \label{fig:exp2}
  % \vspace{-15pt}
\end{figure*}

% \begin{table}[!t]
% \centering
% \scriptsize
% \caption{\textbf{KoNViD-1k runtime and memory overhead.} We evaluate per-iteration runtime and VRAM usage for KVQ with different batch sizes. B2P-Corr adds negligible time and VRAM overhead, while improving test PCC relative to the naive mini-batch correlation objective.}
% \setlength{\tabcolsep}{1.5pt}
% \begin{tabular}{l|cc|cc|cc}
% \toprule
% \multicolumn{1}{l|}{\textbf{Batch-size}} &
% \multicolumn{2}{c|}{\textbf{2 (default)}} & 
% \multicolumn{2}{c|}{\textbf{4}} & 
% \multicolumn{2}{c}{\textbf{8}} \\
% \cmidrule(lr){1-1} \cmidrule(lr){2-3} \cmidrule(lr){4-5} \cmidrule(lr){6-7}
% \textbf{Overhead} 
% & \textbf{Sec/iter} & VRAM  
% & \textbf{Sec/iter} & VRAM
% & \textbf{Sec/iter} & VRAM \\
% \midrule
% KVQ & 3.168 & 25.102 GB & 5.730 & 44.914 GB & 11.061 & 76.216 GB \\
% +\textbf{B2P-EMA} & 3.169 & 25.102 GB & 5.734 & 44.914 GB & 11.093 & 76.216 GB \\
% +\textbf{B2P-U} & 3.177 & 25.102 GB & 5.764 & 44.914 GB & 11.165 & 76.216 GB \\
% \bottomrule
% \end{tabular}
% \vspace{-10px}
% \label{tab:overhead}
% \end{table}

\subsubsection{Result Analysis}
\label{sssec:realworld_results}

\Cref{tab:real_results} shows that replacing the naive mini-batch PCC/CCC term with B2P-EMA/B2P-U consistently improves correlation across all three tasks and batch sizes, while keeping architectures, optimizers, preprocessing, and auxiliary losses unchanged.

\textbf{Video Quality Assessment (KoNViD-1k).} In the unstable small-batch regime ($B{=}2$), B2P-EMA yields substantial gains, boosting PCC from $0.865$ to $0.895$ and reducing RMSE by $6.4\%$ ($0.303 \to 0.283$). This significantly outperforms the fine-grain correlation loss PLC ($0.871$), confirming that the primary bottleneck is the bias in local moment estimation rather than model capacity. As batch size increases, B2P-Corr variants maintain superiority, with B2P-U achieving the best overall performance at $B=8$ (PCC $0.906$, SROCC $0.893$, RMSE $0.264$). This aligns with our analysis that naive mini-batch correlation gradients couple per-example terms with denominators (Proposition~\ref{prop:bias}), and our B2P-U effectiveness reduces variance during moment estimation.

\begin{table*}[t]
\centering
\fontsize{6.1}{6.5}\selectfont 
\setlength{\tabcolsep}{1pt}
\caption{Comparison across three datasets. We replace naive mini-batch PCC/CCC objectives in each baseline with our B2P-Corr and it consistently improves performance. ``Default" denotes the original batch size used in each baseline. The VEATIC arousal prediction result are shown in \Cref{app:veatic_arousal}.}
\begin{tabular}{c|l|ccc|ccc|ccc}  % Changed first 'l' to 'c' for horizontal centering
\toprule
\multirow{2}{*}{\textbf{Dataset}} & \multicolumn{1}{c|}{\textbf{Batch-size}} &
\multicolumn{3}{c|}{\textbf{2 (Default)}} & 
\multicolumn{3}{c|}{\textbf{4}} & 
\multicolumn{3}{c}{\textbf{8}} \\
\cmidrule(lr){2-2} \cmidrule(lr){3-5} \cmidrule(lr){6-8} \cmidrule(lr){9-11}
& \multicolumn{1}{c|}{\textbf{Metrics}} 
& \textbf{PCC}~$\uparrow$ & \textbf{SROCC}~$\uparrow$ & \textbf{RMSE}~$\downarrow$
& \textbf{PCC}~$\uparrow$ & \textbf{SROCC}~$\uparrow$ & \textbf{RMSE}~$\downarrow$
& \textbf{PCC}~$\uparrow$ & \textbf{SROCC}~$\uparrow$ & \textbf{RMSE}~$\downarrow$ \\
\midrule
\multirow{4}{*}[-2pt]{\rotatebox[origin=c]{90}{\shortstack{KoNViD-1k\\(PCC)}}} & KVQ \citep{qu2025kvq} & 0.865$_{\scriptstyle\pm0.028}$& 0.861$_{\scriptstyle\pm0.023}$& 0.303$_{\scriptstyle\pm0.018}$& 0.878$_{\scriptstyle\pm0.029}$& 0.863$_{\scriptstyle\pm0.032}$& 0.299$_{\scriptstyle\pm0.017}$& 0.881$_{\scriptstyle\pm0.033}$& 0.875$_{\scriptstyle\pm0.026}$& 0.290$_{\scriptstyle\pm0.015}$\\
& \hspace{2pt}+PLC \citep{chen2022plc} & 0.871$_{\scriptstyle\pm0.033}$& 0.864$_{\scriptstyle\pm0.027}$& 0.301$_{\scriptstyle\pm0.016}$& 0.881$_{\scriptstyle\pm0.031}$& 0.853$_{\scriptstyle\pm0.021}$& 0.289$_{\scriptstyle\pm0.010}$& 0.887$_{\scriptstyle\pm0.028}$& 0.871$_{\scriptstyle\pm0.035}$& 0.293$_{\scriptstyle\pm0.015}$\\
\cmidrule(lr){2-11}
& \textbf{\hspace{2pt}+B2P-EMA} & \textbf{0.895}$_{\scriptstyle\pm0.024}$& \textbf{0.880}$_{\scriptstyle\pm0.030}$& \textbf{0.283}$_{\scriptstyle\pm0.011}$& \underline{0.893}$_{\scriptstyle\pm0.023}$& \textbf{0.884}$_{\scriptstyle\pm0.019}$& \underline{0.279}$_{\scriptstyle\pm0.017}$& \underline{0.895}$_{\scriptstyle\pm0.026}$& \underline{0.881}$_{\scriptstyle\pm0.022}$& \underline{0.280}$_{\scriptstyle\pm0.010}$\\
& \textbf{\hspace{2pt}+B2P-U} & \underline{0.887}$_{\scriptstyle\pm0.024}$& \underline{0.867}$_{\scriptstyle\pm0.018}$& \underline{0.289}$_{\scriptstyle\pm0.013}$& \textbf{0.896}$_{\scriptstyle\pm0.028}$& \underline{0.881}$_{\scriptstyle\pm0.021}$& \textbf{0.278}$_{\scriptstyle\pm0.013}$& \textbf{0.906}$_{\scriptstyle\pm0.025}$& \textbf{0.893}$_{\scriptstyle\pm0.028}$& \textbf{0.264}$_{\scriptstyle\pm0.010}$\\
\midrule
\midrule
\multirow{2}{*}{\textbf{Dataset}} & \multicolumn{1}{c|}{\textbf{Batch-size}} &
\multicolumn{3}{c|}{\textbf{8}} & 
\multicolumn{3}{c|}{\textbf{16}} & 
\multicolumn{3}{c}{\textbf{32 (Default)}} \\
\cmidrule(lr){2-2} \cmidrule(lr){3-5} \cmidrule(lr){6-8} \cmidrule(lr){9-11}
& \multicolumn{1}{c|}{\textbf{Metrics}} 
& \textbf{PCC}~$\uparrow$ & \textbf{MAE}~$\downarrow$ & \textbf{MSE}~$\downarrow$
& \textbf{PCC}~$\uparrow$ & \textbf{MAE}~$\downarrow$ & \textbf{MSE}~$\downarrow$
& \textbf{PCC}~$\uparrow$ & \textbf{MAE}~$\downarrow$ & \textbf{MSE}~$\downarrow$ \\
\midrule
\multirow{4}{*}[0pt]{\rotatebox[origin=c]{90}{\shortstack{10xProteomic\\(PCC)}}} & EGN \citep{yang2023egn} & 0.615$_{\scriptstyle\pm0.094}$ & 0.166$_{\scriptstyle\pm0.085}$ & 0.056$_{\scriptstyle\pm0.052}$ & 0.637$_{\scriptstyle\pm0.150}$ & \underline{0.161}${\scriptstyle\pm0.091}$ & 0.057$_{\scriptstyle\pm0.059}$ & 0.629$_{\scriptstyle\pm0.150}$ & 0.165$_{\scriptstyle\pm0.098}$ & 0.061$_{\scriptstyle\pm0.066}$ \\
& \hspace{2pt}+PLC \citep{chen2022plc} & 0.623$_{\scriptstyle\pm0.097}$ & 0.165$_{\scriptstyle\pm0.087}$ & 0.055$_{\scriptstyle\pm0.049}$ & 0.645$_{\scriptstyle\pm0.139}$ & 0.162$_{\scriptstyle\pm0.101}$ & 0.059$_{\scriptstyle\pm0.064}$ & 0.631$_{\scriptstyle\pm0.149}$ & 0.169$_{\scriptstyle\pm0.110}$ & 0.061$_{\scriptstyle\pm0.063}$ \\
\cmidrule(lr){2-11}
& \hspace{2pt}\textbf{+B2P-EMA} & \underline{0.644}${\scriptstyle\pm0.089}$ & \underline{0.159}${\scriptstyle\pm0.083}$ & \underline{0.052}${\scriptstyle\pm0.049}$ & \textbf{0.675$_{\scriptstyle\pm0.125}$} & 0.162$_{\scriptstyle\pm0.103}$ & \underline{0.056}${\scriptstyle\pm0.061}$ & \underline{0.665$_{\scriptstyle\pm0.144}$}& \underline{0.159$_{\scriptstyle\pm0.101}$}& \underline{0.055$_{\scriptstyle\pm0.059}$}\\
& \hspace{2pt}\textbf{+B2P-U} & \textbf{0.660$_{\scriptstyle\pm0.091}$} & \textbf{0.155$_{\scriptstyle\pm0.080}$} & \textbf{0.051$_{\scriptstyle\pm0.048}$} & \underline{0.667}${\scriptstyle\pm0.134}$ & \textbf{0.158$_{\scriptstyle\pm0.093}$} & \textbf{0.053$_{\scriptstyle\pm0.057}$} & \textbf{0.673}${\scriptstyle\pm0.161}$& \textbf{0.155}${\scriptstyle\pm0.087}$& \textbf{0.051}${\scriptstyle\pm0.073}$\\
\midrule
\midrule
\multirow{2}{*}{\textbf{Dataset}} & \multicolumn{1}{c|}{\textbf{Batch-size}} &
\multicolumn{3}{c|}{\textbf{4}} & 
\multicolumn{3}{c|}{\textbf{10}} & 
\multicolumn{3}{c}{\textbf{20 (Default)}} \\
\cmidrule(lr){2-2} \cmidrule(lr){3-5} \cmidrule(lr){6-8} \cmidrule(lr){9-11}
& \multicolumn{1}{c|}{\textbf{Metrics}} 
& \textbf{CCC}~$\uparrow$ & \textbf{PCC}~$\uparrow$ & \textbf{RMSE}~$\downarrow$
& \textbf{CCC}~$\uparrow$ & \textbf{PCC}~$\uparrow$ & \textbf{RMSE}~$\downarrow$
& \textbf{CCC}~$\uparrow$ & \textbf{PCC}~$\uparrow$ & \textbf{RMSE}~$\downarrow$ \\
\midrule
\multirow{3}{*}[-2pt]{\rotatebox[origin=c]{90}{\shortstack{VEATIC\\Valence\\(CCC)}}} & VEATIC \citep{ren2024veatic} & 0.437$_{\scriptstyle\pm0.036}$& 0.520$_{\scriptstyle\pm0.029}$& 0.311$_{\scriptstyle\pm0.015}$& 0.574$_{\scriptstyle\pm0.045}$& 0.638$_{\scriptstyle\pm0.028}$& 0.294$_{\scriptstyle\pm0.019}$& 0.606$_{\scriptstyle\pm0.041}$& 0.639$_{\scriptstyle\pm0.038}$& 0.317$_{\scriptstyle\pm0.012}$\\
\cmidrule(lr){2-11}
& \hspace{2pt}\textbf{+B2P-EMA} & \underline{0.460}$_{\scriptstyle\pm0.035}$& \textbf{0.541}$_{\scriptstyle\pm0.037}$& \underline{0.306}$_{\scriptstyle\pm0.013}$& \underline{0.585}$_{\scriptstyle\pm0.040}$& \underline{0.644}$_{\scriptstyle\pm0.033}$& \underline{0.291}$_{\scriptstyle\pm0.015}$& \underline{0.613}$_{\scriptstyle\pm0.035}$& \underline{0.660}$_{\scriptstyle\pm0.033}$& \textbf{0.298}$_{\scriptstyle\pm0.009}$\\
& \hspace{2pt}\textbf{+B2P-U} & \textbf{0.467}$_{\scriptstyle\pm0.025}$& \underline{0.538}$_{\scriptstyle\pm0.028}$& \textbf{0.298}$_{\scriptstyle\pm0.010}$& \textbf{0.597}$_{\scriptstyle\pm0.043}$& \textbf{0.655}$_{\scriptstyle\pm0.031}$& \textbf{0.287}$_{\scriptstyle\pm0.015}$& \textbf{0.615}$_{\scriptstyle\pm0.026}$& \textbf{0.671}$_{\scriptstyle\pm0.022}$& \underline{0.301}$_{\scriptstyle\pm0.008}$\\
\bottomrule
\end{tabular}
% \vspace{-12pt}
\label{tab:real_results}
\end{table*}

\textbf{Spatial Transcriptomic (10xProteomic).} \Cref{tab:real_results} reports 10xProteomic gene-expression prediction results under 3-fold evaluation. The result shows that replacing the naive PCC term in EGN's objective with B2P-Corr yields consistent improvements. In the high-variance setting ($B{=}8$), B2P-U is most effective, improving PCC ($0.615 \to 0.660$) while simultaneously reducing MAE and MSE. This indicates that stable correlation optimization enhances rather than compromises magnitude metrics. At the default $B{=}32$, B2P-U scales effectively, achieving the highest PCC ($0.673$) and lowest MSE. Overall, these results highlight B2P-Corr's robustness in heterogeneous, high-noise regimes where reliable global moment estimates are critical.

\paragraph{Video-based Sentiment Analysis (VEATIC).}

\Cref{tab:real_results} also demonstrates B2P-Corr effectively improves and stabilizes CCC optimization, which is more sensitive to moment mismatch than PCC because its normalization term
$D$ couples variance and mean mismatch in addition to covariance. At $B{=}4$, B2P-U improves CCC ($0.437 \to 0.467$) and RMSE ($0.311 \to 0.298$). B2P-U remains dominant at $B{=}10$, yielding the best CCC ($0.597$) and lowest RMSE ($0.287$). Even at the default $B{=}20$, B2P-U boosts CCC from $0.606$ to $0.615$, verifying that decoupling the gradient calculation from noisy batch denominators prevents the estimation error propagation.

\begin{wraptable}{r}{0.55\textwidth}  % {r} = right side, 0.55\textwidth = width reserved
\centering
\scriptsize
\vspace{-32pt}  % optional: tighten vertical alignment with surrounding text
\caption{\textbf{KoNViD-1k runtime and memory overhead.} We evaluate per-iteration runtime and VRAM usage for KVQ with different batch sizes. B2P-Corr adds negligible time and VRAM overhead, while improving test PCC relative to naive mini-batch correlation objective.}
\setlength{\tabcolsep}{1.5pt}
\begin{tabular}{l|cc|cc|cc}
\toprule
\multicolumn{1}{l|}{\textbf{Batch-size}} &
\multicolumn{2}{c|}{\textbf{2 (default)}} & 
\multicolumn{2}{c|}{\textbf{4}} & 
\multicolumn{2}{c}{\textbf{8}} \\
\cmidrule(lr){1-1} \cmidrule(lr){2-3} \cmidrule(lr){4-5} \cmidrule(lr){6-7}
\textbf{Overhead} 
& \textbf{Sec/iter} & VRAM  
& \textbf{Sec/iter} & VRAM
& \textbf{Sec/iter} & VRAM \\
\midrule
KVQ & 3.168 & 25.102 GB & 5.730 & 44.914 GB & 11.061 & 76.216 GB \\
+\textbf{B2P-EMA} & 3.169 & 25.102 GB & 5.734 & 44.914 GB & 11.093 & 76.216 GB \\
+\textbf{B2P-U} & 3.177 & 25.102 GB & 5.764 & 44.914 GB & 11.165 & 76.216 GB \\
\bottomrule
\end{tabular}
\vspace{-8px}
\label{tab:overhead}
\end{wraptable}

\subsection{Efficiency and Sensitivity}
\label{ssec:sensitivity_and_efficiency}

% \Cref{tab:overhead} reports the computational cost on KoNViD-1k. B2P-EMA introduces negligible overhead (scalar updates). B2P-U introduces a marginal increase in runtime ($<1\%$) due to the reservoir and cross-batch pairing, yet remains highly efficient with no additional VRAM usage.

% Hyperparameter analysis for the B2P-EMA update rate $\alpha$ (using 10xProteomic dataset) and the B2P-U reservoir size (using VEATIC dataset) are provided in \Cref{app:alpha_sensitivity,app:reservoir-sensitivity}.
% Both variants are stable over a practical range. For B2P-EMA, larger batches tolerate larger $\alpha$ since per-step batch statistics are closer to population moments; For B2P-U, the reservoir size controls a trade-off between moment-estimation variance and overhead: a larger reservoir reduces estimation noise but increases the risk of stale cached predictions and higher runtime. Empirically, reservoir size of $2B$ to $8B$ work well for batch size of $B$.

\Cref{tab:overhead} shows B2P-EMA adds essentially no runtime or VRAM cost, while B2P-U incurs less than $1\%$ extra runtime from reservoir updates and cross-batch pairing with no additional VRAM. Sensitivity to $\alpha$ and reservoir size is reported in \Cref{app:alpha_sensitivity,app:reservoir-sensitivity}. Both variants remain stable across practical ranges. Empirically, for B2P-EMA, larger batches tolerate larger $\alpha$ since batch statistics are closer to population moments. For B2P-U, a larger reservoir reduces estimation noise but increases higher runtime.

\section{Conclusion}

% We introduced \textbf{B2P-Corr}, bridging the batch-to-population gap in correlation learning by optimizing decomposable stop-gradient pseudo-losses with lightweight global moment sketches. \textbf{B2P-EMA} yields asymptotically consistent PCC/CCC gradients with negligible overhead, while \textbf{B2P-U} further reduces variance via incomplete U-statistics. We quantify the $O(1/B)$ bias of naive mini-batch PCC/CCC gradients and prove consistency/variance-reduction guarantees for our estimators. Across synthetic and real-world datasets, B2P-Corr stabilizes training and improves correlation with negligible overhead. By offering a statistically grounded and practical drop-in replacement, B2P-Corr establishes a robust path for correlation-driven deep learning.
% We introduced \textbf{B2P-Corr}, bridging the batch-to-population gap in correlation learning via decomposable stop-gradient pseudo-losses and global moment sketches. \textbf{B2P-EMA} achieves asymptotic gradient tracking, while \textbf{B2P-U} uses sampled cross-batch U-statistics and reservoir-based stabilization in an $O(B)$ implementation. We quantify the $O(1/B)$ naive mini-batch bias, prove a tracking result for B2P-EMA, and provide an idealized control-variate analysis for B2P-U. Across synthetic and real-world datasets, this principled drop-in replacement stabilizes training and improves correlation with negligible overhead.

We introduced \textbf{B2P-Corr}, a simple framework for closing the batch-to-population gap in correlation learning with decomposable stop-gradient pseudo-losses and lightweight moment sketches. \textbf{B2P-EMA} provides asymptotic gradient tracking, while \textbf{B2P-U} adds sampled cross-batch U-statistics with reservoir-based stabilization. Across synthetic and real-world tasks, B2P-Corr reduces the batch-to-population gap, stabilizes training, and improves correlation with negligible overhead.

\clearpage

\section{Acknowledgement}
This work was partially supported by US National Science Foundation IIS-2412195, CCF-2400785, the Cancer Prevention and Research Institute of Texas (CPRIT) award (RP230363), the National Institutes of Health (NIH) R01 award (1R01AI190103-01) and Microsoft Accelerate Foundation Models Research (2024).

% \section*{References}

% \medskip
% {
% \small
% % [1] Alexander, J.A.\ \& Mozer, M.C.\ (1995) Template-based algorithms for
% % connectionist rule extraction. In G.\ Tesauro, D.S.\ Touretzky and T.K.\ Leen
% % (eds.), {\it Advances in Neural Information Processing Systems 7},
% % pp.\ 609--616. Cambridge, MA: MIT Press.
% }

\bibliography{references}
\bibliographystyle{plainnat}

%%%%%%%%%%%%%%%%%%%%%%%%%%%%%%%%%%%%%%%%%%%%%%%%%%%%%%%%%%%%

\clearpage

\appendix

\begin{center}
    \Large \textbf{Appendix}
\end{center}

\allowdisplaybreaks

\section{Reproducibility}
\label{app:reproduce}

We provide a PyTorch~\citep{paszke2019pytorch} implementation of B2P-EMA and B2P-U in \href{https://github.com/jyan97/b2p-corr}{Github}. B2P-EMA and B2P-U are implemented as \texttt{torch.nn.Module} object, enabling a simple drop-in replacement of the naive PCC/CCC loss. An example usage is shown below:
\begin{lstlisting}[style=b2pcode]
from b2p import B2PCorrEMA, B2PCorrU

pcc_loss_ema = B2PCorrEMA(metric="PCC", alpha=0.05)
pcc_loss_u = B2PCorrU(metric="CCC", alpha=0.05, reservoir_size=256)

loss = pcc_loss_ema(y_label, y_predict)
loss.backward()
\end{lstlisting}

\section{Additional Proofs and Derivations}
\label{sec:appendix}

\subsection{Useful identities: U-statistics}
\label{sec:appendix-ustat}

\begin{lemma}[U-statistic identities used in \S\ref{sec:b2p-u}]
\label{lem:ustat}
Let $(y_1,f_1)$ and $(y_2,f_2)$ be i.i.d.\ copies of $(y,f_\theta(x))$ with $\E[y^2]<\infty$ and $\E[f^2]<\infty$.
Then
\begin{align}
\Cov(y,f) &= \frac12\,\E\big[(y_1-y_2)(f_1-f_2)\big], \label{eq:app-u-cov}\\
\Var(f)   &= \frac12\,\E\big[(f_1-f_2)^2\big]. \label{eq:app-u-var}
\end{align}
\end{lemma}

\begin{proof}
We expand and use independence.
For covariance,
\begin{align*}
\E[(y_1-y_2)(f_1-f_2)]
&= \E[y_1f_1] - \E[y_1f_2] - \E[y_2f_1] + \E[y_2f_2]\\
&= \E[yf] - \E[y_1]\E[f_2] - \E[y_2]\E[f_1] + \E[yf]\\
&= 2\big(\E[yf]-\E[y]\E[f]\big) \;=\; 2\,\Cov(y,f),
\end{align*}
which gives \Cref{eq:app-u-cov}.
Similarly,
\begin{align*}
\E[(f_1-f_2)^2]
&= \E[f_1^2] - 2\E[f_1f_2] + \E[f_2^2]
= 2\E[f^2] - 2(\E[f])^2
= 2\,\Var(f),
\end{align*}
giving \Cref{eq:app-u-var}.
\end{proof}

\subsection{Proof of Lemma~\ref{lem:pcc-grad} (population gradient of PCC)}
\label{sec:appendix-proof-pcc}

\begin{proof}[Proof of Lemma~\ref{lem:pcc-grad}]
Assume Assumption~\ref{assump:standing} and write $f=f_\theta(x)$.
Recall
\[
\rho(\theta) \;=\; \frac{C}{\sigma_y\sigma_f},
\qquad
C=\E[(y-\mu_y)(f-\mu_f)],
\qquad
\sigma_f = \sqrt{\Var(f)}.
\]
Since $\E[y-\mu_y]=0$, we may rewrite
\[
C = \E[(y-\mu_y)f].
\]
Using differentiation under the expectation,
\begin{equation}
\label{eq:app-dC}
\nabla_\theta C
= \nabla_\theta \E[(y-\mu_y)f]
= \E[(y-\mu_y)\,\nabla_\theta f].
\end{equation}

Next, $\sigma_f = (\sigma_f^2)^{1/2}$ with $\sigma_f^2=\Var(f)=\E[f^2]-(\E[f])^2$.
Again interchanging $\nabla_\theta$ and $\E$ gives
\begin{align}
\nabla_\theta \sigma_f^2
&= \nabla_\theta \E[f^2] - 2\mu_f \nabla_\theta \mu_f
= \E[2f\,\nabla_\theta f] - 2\mu_f\,\E[\nabla_\theta f]
= 2\,\E[(f-\mu_f)\nabla_\theta f]. \label{eq:app-dvar}
\end{align}
Therefore,
\begin{equation}
\label{eq:app-dsig}
\nabla_\theta \sigma_f
= \frac{1}{2\sigma_f}\nabla_\theta \sigma_f^2
= \frac{1}{\sigma_f}\,\E[(f-\mu_f)\nabla_\theta f].
\end{equation}

Now apply the quotient rule to $\rho=C/(\sigma_y\sigma_f)$, noting $\sigma_y$ is constant in $\theta$:
\begin{align*}
\nabla_\theta \rho
&= \frac{1}{\sigma_y}\left(\frac{\nabla_\theta C}{\sigma_f} - \frac{C}{\sigma_f^2}\nabla_\theta\sigma_f\right)\\
&= \frac{1}{\sigma_y}\left(\frac{1}{\sigma_f}\E[(y-\mu_y)\nabla_\theta f]
-\frac{C}{\sigma_f^2}\cdot\frac{1}{\sigma_f}\E[(f-\mu_f)\nabla_\theta f]\right)\\
&= \E\!\left[\left(\frac{y-\mu_y}{\sigma_y\sigma_f}-\frac{C}{\sigma_y\sigma_f}\frac{f-\mu_f}{\sigma_f^2}\right)\nabla_\theta f\right].
\end{align*}
Since $\rho=C/(\sigma_y\sigma_f)$, we obtain
\[
\nabla_\theta \rho
= \E\Bigg[
\Bigg(
\frac{y-\mu_y}{\sigma_y \sigma_f}
-
\rho\,\frac{f-\mu_f}{\sigma_f^2}
\Bigg)\nabla_\theta f
\Bigg],
\]
which is exactly \Cref{eq:pcc-grad}.
\end{proof}

\subsection{Proof of Lemma~\ref{lem:ccc-grad} (population gradient of CCC)}
\label{sec:appendix-proof-ccc}

\begin{proof}[Proof of Lemma~\ref{lem:ccc-grad}]
Assume Assumption~\ref{assump:standing} and write $f=f_\theta(x)$.
Recall
\[
\mathrm{CCC}(\theta)=\frac{2C}{D},
\qquad
D=\sigma_y^2+\sigma_f^2+(\mu_y-\mu_f)^2.
\]
We already have $\nabla_\theta C=\E[(y-\mu_y)\nabla_\theta f]$ from \Cref{eq:app-dC}.

Next compute $\nabla_\theta D$. Since $\sigma_y^2$ and $\mu_y$ are constants in $\theta$,
\[
\nabla_\theta D
= \nabla_\theta \sigma_f^2 + \nabla_\theta(\mu_y-\mu_f)^2.
\]
From \Cref{eq:app-dvar}, $\nabla_\theta\sigma_f^2=2\E[(f-\mu_f)\nabla_\theta f]$.
Also, $(\mu_y-\mu_f)^2=(\mu_f-\mu_y)^2$ so
\[
\nabla_\theta(\mu_y-\mu_f)^2
= 2(\mu_f-\mu_y)\,\nabla_\theta\mu_f
= 2(\mu_f-\mu_y)\,\E[\nabla_\theta f].
\]
Combining and regrouping,
\begin{align}
\nabla_\theta D
&= 2\E[(f-\mu_f)\nabla_\theta f] + 2(\mu_f-\mu_y)\E[\nabla_\theta f] \nonumber\\
&= 2\E[(f-\mu_y)\nabla_\theta f]. \label{eq:app-dD}
\end{align}

Finally, apply the quotient rule to $2C/D$:
\begin{align*}
\nabla_\theta \mathrm{CCC}
&= \frac{2}{D^2}\Big(D\,\nabla_\theta C - C\,\nabla_\theta D\Big)\\
&= \frac{2}{D^2}\Big(D\,\E[(y-\mu_y)\nabla_\theta f] - C\cdot 2\E[(f-\mu_y)\nabla_\theta f]\Big)\\
&= \E\Bigg[
\frac{2}{D^2}\Big(D(y-\mu_y)-2C(f-\mu_y)\Big)\,\nabla_\theta f
\Bigg],
\end{align*}
which is exactly \Cref{eq:ccc-grad}.
\end{proof}

\section{Proofs for Theoretical Studies}
\label{sec:appendix-proofs}

\subsection{Pseudo-loss gradients recover the population gradients}
\label{sec:appendix-proof-pseudoloss}

\begin{lemma}[Pseudo-losses match population gradients when moments are correct]
\label{lem:pseudoloss-matches}
Let $\mathcal{L}_\rho$ and $\mathcal{L}_{\mathrm{CCC}}$ be defined in \Cref{eq:pseudo-pcc,eq:pseudo-ccc}.
If the population moments $(\mu_y,\sigma_y,\mu_f,\sigma_f,C,D,\rho)$ are used inside $\operatorname{sg}(\cdot)$,
then
\[
\nabla_\theta \E[\mathcal{L}_\rho(f_\theta(x);y)] = -\nabla_\theta \rho(\theta),
\qquad
\nabla_\theta \E[\mathcal{L}_{\mathrm{CCC}}(f_\theta(x);y)] = -\nabla_\theta \mathrm{CCC}(\theta).
\]
\end{lemma}

\begin{proof}
All moments are wrapped in $\operatorname{sg}(\cdot)$ and thus treated as constants in backpropagation. Hence,
\[
\nabla_\theta \E[\mathcal{L}_\rho(f;y)]
= \E\!\left[\frac{\partial \mathcal{L}_\rho}{\partial f}\,\nabla_\theta f\right].
\]
Differentiating \Cref{eq:pseudo-pcc} w.r.t.\ $f$ gives
\[
\frac{\partial \mathcal{L}_\rho}{\partial f}
= -\left(\frac{y-\mu_y}{\sigma_y\sigma_f}-\rho\,\frac{f-\mu_f}{\sigma_f^2}\right),
\]
which matches the negative of the per-example weight in Lemma~\ref{lem:pcc-grad}. Thus
$\nabla_\theta \E[\mathcal{L}_\rho]=-\nabla_\theta\rho$.
The CCC case is identical: differentiating \Cref{eq:pseudo-ccc} yields
\[
\frac{\partial \mathcal{L}_{\mathrm{CCC}}}{\partial f}
= -\frac{2}{D^2}\Big(D(y-\mu_y)-2C(f-\mu_y)\Big),
\]
which matches the negative weight in Lemma~\ref{lem:ccc-grad}.
\end{proof}

%%% 最新版本 不需要^4的高阶矩

\subsection{Proof of Proposition~\ref{prop:bias} (Bias of naive mini-batch gradients)}
\label{sec:appendix-proof-bias}

\begin{proof}[Proof of Proposition~\ref{prop:bias} (PCC case)]
We present the PCC case; CCC follows identically since it is also a smooth ratio of finitely many
raw moments of $(y,f)$ on the non-degenerate domain.

\paragraph{Step 1: PCC as a smooth function of raw moments.}
Fix $\theta$ and write $f=f_\theta(x)$. Define
\[
r(\theta) := (r_1,r_2,r_3)^\top := \big(\E[f],\ \E[f^2],\ \E[yf]\big)^\top.
\]
On the domain $\{r: r_2-r_1^2>0\}$ define
\[
\varphi_\rho(r) := \frac{r_3-\mu_y r_1}{\sigma_y\sqrt{r_2-r_1^2}},
\qquad\text{so that}\qquad
\rho(\theta)=\varphi_\rho(r(\theta)).
\]
Given an i.i.d.\ mini-batch $\mathcal{B}=\{(x_i,y_i)\}_{i=1}^B$,
let $s_i=(f_i,f_i^2,y_i f_i)^\top$ and $\hat r = \frac1B\sum_{i=1}^B s_i$, so
$\widehat\rho_{\mathcal B}=\varphi_\rho(\hat r)$.

Let $J(r):=\nabla_r\varphi_\rho(r)^\top\in\mathbb{R}^{1\times 3}$ and $H(r):=\nabla_r^2\varphi_\rho(r)\in\mathbb{R}^{3\times 3}$.
Under Assumption~\ref{assump:standing}(3), $\varphi_\rho$ is $C^3$ on a neighborhood of $r(\theta)$
and its derivatives up to third order are bounded there.

\paragraph{Step 2: Taylor expansion and decomposition.}
Fix any coordinate $k\in\{1,\dots,d\}$ and write $\partial_k := \partial/\partial \theta_k$.
By the chain rule,
\[
\partial_k \widehat\rho_{\mathcal B} = J(\hat r)\,\partial_k \hat r,
\qquad
\partial_k \hat r=\frac1B\sum_{i=1}^B \partial_k s_i.
\]
Let $\delta r:=\hat r-r$. A second-order Taylor expansion of $J$ around $r$ yields
\[
J(\hat r)=J(r)+\delta r^\top H(r)+R,
\qquad\text{with}\qquad
\|R\|\le C\,\|\delta r\|^2
\]
for some constant $C<\infty$ (depending on third derivatives near $r$). Therefore,
\begin{equation}
\label{eq:bias-decomp-clean}
\partial_k \widehat\rho_{\mathcal B}
= \underbrace{J(r)\,\partial_k \hat r}_{(i)}
+ \underbrace{\delta r^\top H(r)\,\partial_k \hat r}_{(ii)}
+ \underbrace{R\,\partial_k \hat r}_{(iii)}.
\end{equation}

\paragraph{Term (i): unbiased leading term.}
By Assumption~\ref{assump:standing}(1), $\E[\partial_k \hat r]=\partial_k r$. Hence
\[
\E\!\left[J(r)\,\partial_k \hat r\right]=J(r)\,\partial_k r=\partial_k \rho.
\]

\paragraph{Term (ii): the $O(1/B)$ bias term.}
Write $u_i:=s_i-\E[s]\in\mathbb{R}^3$ and $g_i^{(k)}:=\partial_k s_i\in\mathbb{R}^3$. Then
\[
\delta r=\frac1B\sum_{j=1}^B u_j,
\qquad
\partial_k \hat r=\frac1B\sum_{i=1}^B g_i^{(k)}.
\]
Using independence across samples and $\E[u_j]=0$,
\[
\E\!\left[\delta r\,(\partial_k \hat r)^\top\right]
= \frac{1}{B^2}\sum_{j=1}^B\sum_{i=1}^B \E[u_j (g_i^{(k)})^\top]
= \frac{1}{B}\,\E\!\left[(s-\E[s])(\partial_k s)^\top\right].
\]
Since $H(r)$ is bounded on the neighborhood, $\E[\delta r^\top H(r)\,\partial_k \hat r]=O(B^{-1})$, and it is generally nonzero
because $s$ and $\partial_k s$ come from the same sample.

\paragraph{Term (iii): the Taylor remainder contributes $O(1/B)$.}
Using $\|R\|\le C\|\delta r\|^2$ and the triangle inequality,
\[
\big|\E[R\,\partial_k \hat r]\big|
\le \E[\|R\|\,\|\partial_k \hat r\|]
\le C\,\E\big[\|\delta r\|^2\,\|\partial_k \hat r\|\big].
\]
Let $G^{(k)}:=\partial_k \hat r=\frac1B\sum_{i=1}^B g_i^{(k)}$. Then
\[
\E\big[\|\delta r\|^2\,\|G^{(k)}\|\big]
\le \frac1B\sum_{i=1}^B \E\big[\|\delta r\|^2\,\|g_i^{(k)}\|\big]
= \E\big[\|\delta r\|^2\,\|g_1^{(k)}\|\big]
\]
by exchangeability. Expanding $\|\delta r\|^2$ using $\delta r=\frac1B\sum_{j=1}^B u_j$ gives
\[
\|\delta r\|^2
= \frac{1}{B^2}\sum_{j=1}^B\sum_{\ell=1}^B \langle u_j,u_\ell\rangle.
\]
Taking expectation against $\|g_1^{(k)}\|$ and using independence and $\E[u_j]=0$,
all mixed terms vanish, yielding
\[
\E\big[\|\delta r\|^2\,\|g_1^{(k)}\|\big]
= \frac{B-1}{B^2}\,\E[\|u\|^2]\,\E[\|g^{(k)}\|]
\;+\; \frac{1}{B^2}\,\E[\|u\|^2\,\|g^{(k)}\|]
= O(B^{-1}),
\]
provided the expectations are finite. This holds under Assumption~\ref{assump:standing}(2)
since by triangle inequality $\|u\|^2 = \|s - \E[s]\|^2 \lesssim 1+\|s\|^2$ and $\|g^{(k)}\|\le \|\nabla_\theta s\|_F$.

Consequently, $\E[R\,\partial_k \hat r]=O(B^{-1})$.

\paragraph{Conclusion.}
Taking expectations in \Cref{eq:bias-decomp-clean} yields
\[
\E_{\mathcal B}\big[\partial_k \widehat\rho_{\mathcal B}\big]
= \partial_k \rho + O(B^{-1}).
\]
Since this holds for each coordinate $k$, we obtain
$\E_{\mathcal B}[\nabla_\theta \widehat\rho_{\mathcal B}] = \nabla_\theta \rho + O(B^{-1})$.
% The CCC case follows identically by defining $\varphi_{\mathrm{CCC}}$ as the smooth mapping from $r$
% to $\mathrm{CCC}=2C/D$ (on $D>0$).
The mini-batch gradient analysis for CCC is identical after defining the corresponding smooth map
\[
\varphi_{\mathrm{CCC}}(r)= \frac{2(r_3-\mu_y r_1)}{\sigma_y^2 + (r_2-r_1^2) + (\mu_y-r_1)^2}
\text{ on the domain } \sigma_y^2 + (r_2-r_1^2) + (\mu_y-r_1)^2>0.
\]
\end{proof}

\subsection{Formal Tracking Regime for Theorem~\ref{prop:ema-consistency}}
\label{app:ema-tracking-assumptions}

\begin{assumption}[B2P-EMA tracking regime]
\label{assump:ema}
Assume Assumption~\ref{assump:standing} holds for all iterates $\theta_t$. Let $m(\theta):=(\mu_f(\theta),\sigma_f^2(\theta),C(\theta))$, $z_t:=(\hat\mu_f^{(t)},\widehat{\sigma_f^2}{}^{(t)},\hat C^{(t)})$, and let $\mathcal F_t$ be all information available before drawing the batch at step $t$. Let $h_t\in\mathbb R^3$ be the detached batch-moment estimate used in \Cref{eq:ema-updates}, and let $\bar g(\theta,z)$ denote the conditional mean gradient of the B2P-EMA pseudo-loss. Assume:
\begin{enumerate}[label=(\arabic*), leftmargin=1.5em]
\item \textbf{Unbiased moment noise:} $\E[h_t\mid \mathcal F_t]=m(\theta_t)$ and $\E[\|h_t-m(\theta_t)\|^2\mid \mathcal F_t]\le V$ for some $V<\infty$.
\item \textbf{EMA rates:} $\alpha_t\in(0,1]$, $\sum_t \alpha_t=\infty$, and $\sum_t \alpha_t^2<\infty$.
\item \textbf{Slow target drift:} $\|m(\theta_{t+1})-m(\theta_t)\|\le \alpha_t \beta_t$ a.s. for some $\beta_t\to 0$.
\item \textbf{Stable plug-in gradient:} for the $\varepsilon$-stabilized PCC/CCC pseudo-losses, $\bar g(\theta_t,\cdot)$ is Lipschitz, i.e.,
$\|\bar g(\theta_t,z)-\bar g(\theta_t,z')\|\le L_g\|z-z'\|$ for some $L_g<\infty$.
\end{enumerate}
\end{assumption}

\subsection{Proof of Theorem~\ref{prop:ema-consistency} (Tracking of B2P-EMA)}
\label{sec:appendix-proof-ema}

\begin{proof}
Write
\[
m_t:=m(\theta_t),\qquad
z_t:=\bigl(\hat\mu_f^{(t)},\widehat{\sigma_f^2}{}^{(t)},\hat C^{(t)}\bigr),\qquad
e_t:=z_t-m_t.
\]

\paragraph{Step 1: Tracking of the moment sketch.}
The EMA recursion can be written as
\[
z_{t+1}=z_t+\alpha_t(h_t-z_t)
      =z_t+\alpha_t\bigl(m_t-z_t+\xi_{t+1}\bigr),
\]
where
\[
\xi_{t+1}:=h_t-m_t.
\]
By Assumption~\ref{assump:ema}(1),
\[
\E[\xi_{t+1}\mid \mathcal F_t]=0,
\qquad
\E[\|\xi_{t+1}\|^2\mid \mathcal F_t]\le V.
\]
Let
\[
\Delta_t:=m_{t+1}-m_t.
\]
Then
\[
e_{t+1}=(1-\alpha_t)e_t+\alpha_t\xi_{t+1}-\Delta_t.
\]

Conditioning on $\mathcal F_t$, we do not pull $\Delta_t$ outside the expectation, since $\theta_{t+1}$ may depend on the same batch used to form $h_t$. Instead, define
\[
u_t:=(1-\alpha_t)e_t+\alpha_t\xi_{t+1},
\qquad
e_{t+1}=u_t-\Delta_t.
\]
Using Young's inequality with $\eta_t=\alpha_t/2$,
\[
\|e_{t+1}\|^2
\le
\Bigl(1+\frac{\alpha_t}{2}\Bigr)\|u_t\|^2
+\Bigl(1+\frac{2}{\alpha_t}\Bigr)\|\Delta_t\|^2.
\]
Since $e_t$ is $\mathcal F_t$-measurable and $\E[\xi_{t+1}\mid \mathcal F_t]=0$,
\begin{align*}
\E[\|u_t\|^2\mid\mathcal F_t]
&=
(1-\alpha_t)^2\|e_t\|^2
+\alpha_t^2\E[\|\xi_{t+1}\|^2\mid\mathcal F_t] \\
&\le
(1-\alpha_t)\|e_t\|^2 + V\alpha_t^2.
\end{align*}
Therefore,
\begin{align*}
\E[\|e_{t+1}\|^2\mid\mathcal F_t]
&\le
\Bigl(1+\frac{\alpha_t}{2}\Bigr)\Bigl((1-\alpha_t)\|e_t\|^2+V\alpha_t^2\Bigr)
+\Bigl(1+\frac{2}{\alpha_t}\Bigr)\|\Delta_t\|^2 \\
&\le
\Bigl(1-\frac{\alpha_t}{2}\Bigr)\|e_t\|^2
+2V\alpha_t^2
+\Bigl(1+\frac{2}{\alpha_t}\Bigr)\|\Delta_t\|^2.
\end{align*}
Applying Assumption~\ref{assump:ema}(3) and $\alpha_t\in(0,1]$,
\[
\Bigl(1+\frac{2}{\alpha_t}\Bigr)\|\Delta_t\|^2
\le
\Bigl(1+\frac{2}{\alpha_t}\Bigr)\alpha_t^2\beta_t^2
\le
3\alpha_t\beta_t^2.
\]
Hence
\[
\E[\|e_{t+1}\|^2\mid\mathcal F_t]
\le
\Bigl(1-\frac{\alpha_t}{2}\Bigr)\|e_t\|^2
+3\alpha_t\beta_t^2
+2V\alpha_t^2.
\]
Taking expectation and writing
\[
a_t:=\E[\|e_t\|^2]
\]
gives
\[
a_{t+1}
\le
\Bigl(1-\frac{\alpha_t}{2}\Bigr)a_t
+3\alpha_t\beta_t^2
+2V\alpha_t^2.
\]

Fix any $\varepsilon>0$. Since $\beta_t\to0$, there exists $T_\varepsilon$ such that
$3\beta_t^2\le \varepsilon$ for all $t\ge T_\varepsilon$. Hence, for $t\ge T_\varepsilon$,
\[
a_{t+1}
\le
\Bigl(1-\frac{\alpha_t}{2}\Bigr)a_t
+\varepsilon\alpha_t
+2V\alpha_t^2.
\]
Equivalently, we have
\[
a_{t+1}-2\varepsilon
\le
\Bigl(1-\frac{\alpha_t}{2}\Bigr)(a_t-2\varepsilon)
+2V\alpha_t^2.
\]
Iterating from $T_\varepsilon$ to $t$ yields
\begin{align*}
a_{t+1}-2\varepsilon
&\le
\prod_{k=T_\varepsilon}^{t}\Bigl(1-\frac{\alpha_k}{2}\Bigr)(a_{T_\varepsilon}-2\varepsilon)
+2V\sum_{k=T_\varepsilon}^{t}\alpha_k^2
\prod_{j=k+1}^{t}\Bigl(1-\frac{\alpha_j}{2}\Bigr).
\end{align*}

The first term vanishes because $\sum_t\alpha_t=\infty$, and the second vanishes because $\sum_t\alpha_t^2<\infty$. Therefore
\[
\limsup_{t\to\infty} a_t \le 2\varepsilon.
\]
Since $\varepsilon>0$ is arbitrary, we conclude
\[
\E\|z_t-m(\theta_t)\|^2 = a_t \to 0.
\]

\paragraph{Step 2: Consistency of the plug-in gradient.}
Let
\[
\bar g_t:=\E[g_t\mid \mathcal F_t]=\bar g(\theta_t,z_t).
\]
By Lemma~\ref{lem:pseudoloss-matches}, plugging the true moments into the stop-gradient pseudo-loss exactly recovers the population gradient:
\[
\bar g(\theta_t,m(\theta_t))=\nabla J(\theta_t).
\]
Using Assumption~\ref{assump:ema}(4) and Jensen's inequality,
\begin{align*}
\E\|\bar g_t-\nabla J(\theta_t)\|
&=
\E\|\bar g(\theta_t,z_t)-\bar g(\theta_t,m(\theta_t))\| \\
&\le
L_g\,\E\|z_t-m(\theta_t)\| \\
&\le
L_g\,\sqrt{\E\|z_t-m(\theta_t)\|^2}
\to 0.
\end{align*}
This proves the asymptotic consistency of the moving-target plug-in gradient.
\end{proof}

\begin{remark}[Constant-$\alpha$ heuristic]
If $\alpha_t\equiv \alpha\in(0,1]$ and
\[
\sup_t \|m(\theta_{t+1})-m(\theta_t)\| \le d,
\]
then the same one-step recursion gives
\[
\limsup_{t\to\infty}\E\|z_t-m(\theta_t)\|^2
\;\le\;
\frac{4d^2}{\alpha^2}+2V\alpha.
\]
Thus a small constant EMA still keeps the sketch accurate when the target moments drift on a slower scale than $\alpha$.
\end{remark}

\subsection{Proof of Proposition~\ref{prop:var-red} (Variance reduction of B2P-U)}
\label{sec:appendix-proof-varred}

\begin{proof}[Proof of Proposition~\ref{prop:var-red}]
Unbiasedness follows from linearity of expectation:
\[
\E[\hat M^\star]
= \E[\hat M_{\mathrm{U}}] + \lambda(\E[\hat M_{\mathcal R}] - \E[\hat M_{\mathcal R}])
= M.
\]
For the variance,
\[
\Var(\hat M^\star)
= \Var(\hat M_{\mathrm{U}}) + \lambda^2 \Var(\hat M_{\mathcal R})
+ 2\lambda\,\Cov(\hat M_{\mathrm{U}},\hat M_{\mathcal R}).
\]
This is a convex quadratic in $\lambda$, minimized at
$\lambda^\star = -\Cov(\hat M_{\mathrm{U}},\hat M_{\mathcal R})/\Var(\hat M_{\mathcal R})$.
Substituting $\lambda^\star$ yields the stated variance reduction, with strict inequality when
$\Cov(\hat M_{\mathrm{U}},\hat M_{\mathcal R})\neq 0$.
\end{proof}

% \subsection{Exact $O(B)$ evaluation of the complete cross-batch estimator}
% \label{app:exact-ob-crossbatch}

% For two batches $\{(y_i,\tilde f_i)\}_{i=1}^B$ and $\{(y'_j,\tilde f'_j)\}_{j=1}^B$, the complete cross-batch estimators over all $B^2$ pairs satisfy
% \begin{align*}
% \hat C_{\mathrm{full}}
% &:= \frac{1}{2B^2}\sum_{i=1}^B\sum_{j=1}^B (y_i-y'_j)(\tilde f_i-\tilde f'_j) \\
% &= \frac12\left(
% \frac1B\sum_{i=1}^B y_i\tilde f_i
% -\bar y_t\,\bar{\tilde f}_{t-1}
% -\bar{\tilde f}_t\,\bar y_{t-1}
% +\frac1B\sum_{j=1}^B y'_j\tilde f'_j
% \right),\\[4pt]
% \widehat{\sigma^2_{f,\mathrm{full}}}
% &:= \frac{1}{2B^2}\sum_{i=1}^B\sum_{j=1}^B (\tilde f_i-\tilde f'_j)^2 \\
% &= \frac12\left(
% \frac1B\sum_{i=1}^B \tilde f_i^2
% -2\bar{\tilde f}_t\,\bar{\tilde f}_{t-1}
% +\frac1B\sum_{j=1}^B {\tilde f'_j}^2
% \right),
% \end{align*}
% where $\bar y_t=\frac1B\sum_{i=1}^B y_i$, $\bar y_{t-1}=\frac1B\sum_{j=1}^B y'_j$, $\bar{\tilde f}_t=\frac1B\sum_{i=1}^B \tilde f_i$, and $\bar{\tilde f}_{t-1}=\frac1B\sum_{j=1}^B \tilde f'_j$. Thus, for the second-order moments used by B2P-U, the complete cross-batch estimator can also be evaluated exactly in $O(B)$ time. Our implementation uses sampled pairs for simplicity and to match the cross-batch sketch update described in \Cref{sec:b2p-u}.

\section{Detailed Experiment Settings}
\label{app:experiment_settings}

\subsection{Synthetic Dataset}
\label{app:synthetic_dataset}

\subsubsection{Exp.~1: Imbalanced mixture regression}

\paragraph{Exp.~1: Imbalanced mixture regression (synthetic).}
We construct an imbalanced two-component Gaussian-mixture regression problem to stress-test correlation objectives under mini-batches.
Let $d=64$ and $N_{\mathrm{pop}}=4000$. We first draw an indicator $z_i\sim\mathrm{Bernoulli}(p_{\mathrm{major}})$ with $p_{\mathrm{major}}=0.9$,
and map it to a sign variable $s_i=2z_i-1\in\{+1,-1\}$, where $s_i=+1$ is the \emph{major} component.
Define a mean vector $\mu\in\mathbb{R}^d$ with $\mu_1=\texttt{mean\_sep}=3.0$ and $\mu_j=0$ for $j\ge2$.
We then sample features
\[
x_i \sim \mathcal{N}(s_i \mu, I_d)\qquad(\text{i.e., }x_i=\xi_i+s_i\mu,\ \xi_i\sim\mathcal{N}(0,I_d)).
\]
Next, we sample a ground-truth direction $w^\star\sim\mathcal{N}(0,I_d)$ and normalize it to unit norm.
The regression target is generated by
\[
y_i = {w^\star}^\top x_i + \Delta\, s_i + \varepsilon_i,\qquad \Delta=\texttt{shift}=8.0,
\]
with heteroscedastic noise
\[
\varepsilon_i \sim \mathcal{N}\!\left(0,\sigma_{s_i}^2\right),\qquad
\sigma_{+1}=\texttt{noise\_major}=1.0,\ \ \sigma_{-1}=\texttt{noise\_minor}=4.0.
\]
This construction yields well-defined population moments, but makes mini-batch estimates of $\mu_f,\Var(f_\theta),\Cov(y,f_\theta)$ highly variable for small batch size $B$ due to (i) strong class imbalance and (ii) component-dependent label shift and noise.

\paragraph{Exp.~1: Gradient-alignment diagnostics (PCC/CCC).}
We treat the entire synthesized dataset $\{(x_i,y_i)\}_{i=1}^{N_{\mathrm{pop}}}$ as the \emph{population} to form oracle moments and oracle gradients.
For a linear predictor $f_\theta(x)=\theta^\top x$, at each checkpoint $\theta(t)$, we compute an oracle gradient $g_{\mathrm{oracle}}(\theta)=\nabla_\theta \rho(\theta)$ (or $\nabla_\theta \mathrm{ccc}(\theta)$) using the full population.
We then estimate stochastic gradients from i.i.d.\ mini-batches of size $B$ sampled from the population:
for each checkpoint we draw $\texttt{n\_trials}=4,096$ batches and compute a per-batch gradient direction
$\hat g = -\nabla_\theta \ell(\theta;\mathcal{B})$, where $\ell$ is the naive mini-batch PCC/CCC or our B2P-Corr surrogate.
We report:
\[
\operatorname{Align} = \cos(\hat g, g_{\mathrm{oracle}}),\qquad
\operatorname{RelBias}=\frac{\left\|\mathbb{E}[\hat g]-g_{\mathrm{oracle}}\right\|_2}{\|g_{\mathrm{oracle}}\|_2},
\]
where expectations are approximated by averaging over the 4096 sampled mini-batches.
We repeat this evaluation for $\texttt{metric}\in\{\operatorname{PCC},\operatorname{CCC}\}$ and $B\in\{2,4,8,16,32\}$.

% =========================
% Hyperparameter Table (from code)
% =========================
\begin{table}[h]
\centering
\caption{Exp.~1 hyperparameters.}
\label{tab:exp1_hparams}
\setlength{\tabcolsep}{6pt}
\begin{tabular}{l l}
\toprule
\multicolumn{2}{c}{\textbf{Data generation}}\\
\midrule
Dimension $d$ & $64$ \\
Population size $N_{\mathrm{pop}}$ & $4000$ \\
Imbalance $p_{\mathrm{major}}=\mathbb{P}(s=+1)$ & $0.9$ \\
Mean separation (\texttt{mean\_sep}) & $3.0$ (only on feature 1; others $0$) \\
Label shift $\Delta$ (\texttt{shift}) & $8.0$ \\
Noise std (major/minor) & $\sigma_{+1}=1.0$, $\sigma_{-1}=4.0$ \\
Data seed & $0$ \\
\midrule
\multicolumn{2}{c}{\textbf{Evaluation protocol (gradient diagnostics)}}\\
\midrule
Metrics & $\{\operatorname{PCC}, \operatorname{CCC}\}$ \\
Batch sizes $B$ & \{2, 4, 8, 16, 32\} \\
Checkpoint grid & 60 points in $[0,1]$ (linear interpolation; renormalized) \\
Mini-batches per checkpoint & $\texttt{n\_trials}=4096$ \\
Random-direction seed (for $\theta_{\mathrm{rand}}$) & \texttt{set\_seed(123)} inside each config run \\
\midrule
\multicolumn{2}{c}{\textbf{B2P-Corr estimator hyperparameters}}\\
\midrule
Shared: step size $\alpha$ & $0.02$ \\
Shared: numerical $\epsilon$ & $10^{-8}$ \\
B2P-U reservoir size & $1024$ \\
\bottomrule
\end{tabular}
\end{table}

\subsubsection{Exp.~2: isotropic linear regression}

\paragraph{Exp.~2: Isotropic linear regression (data generation).}
We construct an isotropic linear regression problem in $\R^d$ with additive Gaussian noise.
Fix the dimension $d=256$ and sample the ground-truth direction $w^\star\sim\mathcal{N}(0,I_d)$, then normalize it to unit norm $\|w^\star\|_2=1$.
Next, we generate i.i.d.\ covariates $x_i\sim\mathcal{N}(0,I_d)$ and responses
\begin{equation}
    y_i = {w^\star}^\top x_i + \varepsilon_i,\qquad \varepsilon_i\sim\mathcal{N}(0,\sigma^2),\ \sigma=3.0,
\end{equation}
for a total of $N_{\text{train}}+N_{\text{test}}=50{,}000+10{,}000$ samples. We split the first $N_{\text{train}}$ examples for training and the remaining $N_{\text{test}}$ for testing.

\paragraph{Exp.~2: Mini-batch PCC optimization.}
We train a bias-free linear predictor $f_\theta(x)=\theta^\top x$ (implemented as \texttt{Linear(d,1,bias=False)}), where $\theta$ is initialized as $\theta\sim\mathcal{N}(0,10^{-4}I_d)$ (std.\ $10^{-2}$).
We optimize the PCC objective using Adam for 5{,}000 steps with mini-batch size $B$ (we run separate experiments for each $B$). We evaluate (i) \emph{DirectionRecovery} $=\cos(\theta,w^\star)$ and (ii) \emph{TestPCC} computed on the \emph{full} test set. Metrics are logged every 100 steps (and at step 1) and averaged over 3 random seeds (which control model initialization and data-loader shuffling).

\begin{table}[h]
\centering
\caption{Exp.~2 hyperparameters}
\label{tab:exp2_hparams}
\setlength{\tabcolsep}{6pt}
\begin{tabular}{l l}
\toprule
\multicolumn{2}{c}{\textbf{Data generation}}\\
\midrule
Dimension $d$ & $256$ \\
Train / test size & $N_{\mathrm{train}}=50{,}000$, $N_{\mathrm{test}}=10{,}000$ \\
Covariates & $x\sim\mathcal{N}(0,I_d)$ \\
Ground-truth direction & $w^\star\sim\mathcal{N}(0,I_d)$, normalized to $\|w^\star\|_2=1$ \\
Response model & $y={w^\star}^\top x+\varepsilon$ \\
Noise & $\varepsilon\sim\mathcal{N}(0,\sigma^2)$ with $\sigma=3.0$ \\
Dataset seed & $0$ \\
\midrule
\multicolumn{2}{c}{\textbf{Training \& evaluation protocol}}\\
\midrule
Initialization & $\theta\sim\mathcal{N}(0,10^{-4}I)$ (\texttt{std}\ =$10^{-2}$) \\
Batch size $B$ & $B\in\{2,4,8,16\}$ \\
Training steps & $5,000$ \\
Optimizer & Adam \\
Learning rate & $5\times 10^{-3}$ \\
Weight decay & $10^{-3}$ \\
Gradient clipping & $\|\nabla\|_2 \le 10$ \\
\midrule
\multicolumn{2}{c}{\textbf{B2P-Corr estimator hyperparameters}}\\
\midrule
Shared: step size $\alpha$ & $0.05$ \\
Shared: numerical $\epsilon$ & $10^{-8}$ \\
B2P-U reservoir size & $256$ \\
\bottomrule
\end{tabular}
\end{table}

\subsection{Real-world Dataset}
\label{app:realworld_dataset}

\subsubsection{Video quality assessment (KoNViD-1k)}

The hyperparameter settings for the KoNViD-1k experiment are provided in~\Cref{tab:setting_kvq}. We strictly follow the data processing and training pipeline of the baseline KVQ~\citep{qu2025kvq}, use the same reported metrics (PCC, SROCC, and RMSE), and only replace the naive mini-batch PCC objective with our decomposable B2P-Corr and B2P-U objectives.

\begin{table}[H]
\centering
\small
\caption{KVQ \citep{qu2025kvq} experiment setting on KoNViD-1k dataset}
\begin{tabular}{l l}
\toprule
\multicolumn{2}{c}{\textbf{Model \& input}}\\
\midrule
Clip length  & $32$ \\
Input channels  & $768$ \\
Hidden channels  & $64$ \\
\midrule
\multicolumn{2}{c}{\textbf{Optimization schedule}}\\
\midrule
Epochs & $60$ \\
Warmup epochs & $1$ \\
Batch size & $B\in\{2,4,8\}$ \\
Optimizer & AdamW \\
Learning rate & $5\times 10^{-4}$ \\
Weight decay & $0.05$ \\
\midrule
\multicolumn{2}{c}{\textbf{B2P-Corr estimator hyperparameters}}\\
\midrule
Shared: step size $\alpha$ & $0.05$ \\
Shared: numerical $\epsilon$ & $10^{-8}$ \\
B2P-U reservoir size & $256$ \\
\bottomrule
\label{tab:setting_kvq}
\end{tabular}
\end{table}

\subsubsection{Spatial transcriptomic (10xProteomic)}

The hyperparameter settings for the 10xProteomic experiment are provided in~\Cref{tab:setting_st}. We strictly follow the data processing and training pipeline of the baseline EGN~\citep{yang2023egn}, use the same reported metrics (PCC, MAE and MSE), and only replace the naive mini-batch PCC objective with our decomposable B2P-Corr and B2P-U objectives.

\begin{table}[H]
\centering
\small
\caption{EGN~\citep{yang2023egn} experiment setting on 10xProteomic dataset}
\begin{tabular}{l l}
\toprule
\multicolumn{2}{c}{\textbf{Model \& input}}\\
\midrule
Number of genes  & $250$ \\
MLP dimension  & $4096$ \\
Hidden dimension & $1024$ \\
Model depth & $8$ \\
\midrule
\multicolumn{2}{c}{\textbf{Optimization schedule}}\\
\midrule
Epochs & $50$ \\
Warmup steps & $100$ \\
Batch size & $B\in\{8,16,32\}$ \\
Optimizer & AdamW \\
Learning rate & $5\times 10^{-4}$ \\
Weight decay & $1\times 10^{-4}$ \\
\midrule
\multicolumn{2}{c}{\textbf{B2P-Corr estimator hyperparameters}}\\
\midrule
Shared: step size $\alpha$ & $0.05$ \\
Shared: numerical $\epsilon$ & $10^{-8}$ \\
B2P-U reservoir size & $256$ \\
\bottomrule
\label{tab:setting_st}
\end{tabular}
\end{table}

\subsubsection{Video-based sentiment analysis (VEATIC)}

The hyperparameter settings for the VEATIC experiment are provided in~\Cref{tab:setting_veatic}. We strictly follow the data processing and training pipeline in VEATIC~\citep{ren2024veatic}, use the same reported metrics (CCC, PCC, RMSE), and only replace the naive mini-batch CCC objective with our decomposable B2P-Corr and B2P-U objectives.

\begin{table}[H]
\centering
\small
\caption{VEATIC~\citep{ren2024veatic} experiment setting}
\begin{tabular}{l l}
\toprule
\multicolumn{2}{c}{\textbf{Model \& input}}\\
\midrule
MLP dimension  & $2048$ \\
Hidden dimension & $2048$ \\
Model depth & $6$ \\
Dropout & $0.4$ \\
Backbone & resnet50 \\
\midrule
\multicolumn{2}{c}{\textbf{Optimization schedule}}\\
\midrule
Epochs & $20$ \\
Batch size & $B\in\{4, 10, 20\}$ \\
Optimizer & Adam \\
Learning rate & $5\times 10^{-4}$ \\
Weight decay & $1\times 10^{-8}$ \\
Betas & $[0.9,0.999]$\\
\midrule
\multicolumn{2}{c}{\textbf{B2P-Corr estimator hyperparameters}}\\
\midrule
Shared: step size $\alpha$ & $0.05$ \\
Shared: numerical $\epsilon$ & $10^{-8}$ \\
B2P-U reservoir size & $256$ \\
\bottomrule
\label{tab:setting_veatic}
\end{tabular}
\end{table}

\subsection{Hardware Resources}
All experiments were conducted on a server running Ubuntu 22.04, equipped with 6 NVIDIA A100 GPUs (80GB VRAM each) and 512GB of system RAM.

\section{Extended Results and Analysis}
\label{app:reseults_and_analysis}

\subsection{VEATIC Dataset Arousal Result}
\label{app:veatic_arousal}

\begin{table*}[h]
\centering
\caption{\textbf{VEATIC arousal (CCC) results.} We keep the baseline architecture, optimizer and auxiliary losses fixed and replace only the naive mini-batch CCC objective with B2P-Corr (B2P-EMA/B2P-U). B2P consistently improves CCC/PCC and reduces RMSE across batch sizes. ``Default'' denotes the batch size used in the original baseline. Results are reported as mean$\pm$std (bold/underline denote best/second best).}
\fontsize{6.1}{6.5}\selectfont 
\setlength{\tabcolsep}{1pt}
\begin{tabular}{c|l|ccc|ccc|ccc}  % Changed first 'l' to 'c' for horizontal centering
\toprule
\multirow{2}{*}{\textbf{Dataset}} & \multicolumn{1}{c|}{\textbf{Batch-size}} &
\multicolumn{3}{c|}{\textbf{4}} & 
\multicolumn{3}{c|}{\textbf{10}} & 
\multicolumn{3}{c}{\textbf{20 (Default)}} \\
\cmidrule(lr){2-2} \cmidrule(lr){3-5} \cmidrule(lr){6-8} \cmidrule(lr){9-11}
& \multicolumn{1}{c|}{\textbf{Metrics}} 
& \textbf{CCC}~$\uparrow$ & \textbf{PCC}~$\uparrow$ & \textbf{RMSE}~$\downarrow$
& \textbf{CCC}~$\uparrow$ & \textbf{PCC}~$\uparrow$ & \textbf{RMSE}~$\downarrow$
& \textbf{CCC}~$\uparrow$ & \textbf{PCC}~$\uparrow$ & \textbf{RMSE}~$\downarrow$ \\
\midrule
\multirow{3}{*}[-2pt]{\rotatebox[origin=c]{90}{\shortstack{VEATIC\\Arousal\\(CCC)}}} & VEATIC \citep{ren2024veatic} & 0.433$_{\scriptstyle\pm0.025}$& 0.515$_{\scriptstyle\pm0.019}$& 0.311$_{\scriptstyle\pm0.017}$& 0.560$_{\scriptstyle\pm0.033}$& 0.602$_{\scriptstyle\pm0.029}$& 0.230$_{\scriptstyle\pm0.012}$& 0.577$_{\scriptstyle\pm0.024}$& 0.619$_{\scriptstyle\pm0.030}$& 0.234$_{\scriptstyle\pm0.015}$\\
\cmidrule(lr){2-11}
& \hspace{2pt}\textbf{+B2P-EMA} & \textbf{0.452}$_{\scriptstyle\pm0.023}$& \underline{0.538}$_{\scriptstyle\pm0.031}$& \underline{0.246}$_{\scriptstyle\pm0.015}$& \underline{0.569}$_{\scriptstyle\pm0.036}$& \underline{0.612}$_{\scriptstyle\pm0.021}$& \underline{0.216}$_{\scriptstyle\pm0.017}$& \underline{0.591}$_{\scriptstyle\pm0.028}$& \underline{0.630}$_{\scriptstyle\pm0.020}$& \underline{0.217}$_{\scriptstyle\pm0.014}$\\
& \hspace{2pt}\textbf{+B2P-U} & \underline{0.450}$_{\scriptstyle\pm0.027}$& \textbf{0.541}$_{\scriptstyle\pm0.019}$& \textbf{0.218}$_{\scriptstyle\pm0.014}$& \textbf{0.578}$_{\scriptstyle\pm0.025}$& \textbf{0.630}$_{\scriptstyle\pm0.033}$& \textbf{0.210}$_{\scriptstyle\pm0.009}$& \textbf{0.603}$_{\scriptstyle\pm0.019}$& \textbf{0.643}$_{\scriptstyle\pm0.023}$& \textbf{0.212}$_{\scriptstyle\pm0.010}$\\
\bottomrule
\end{tabular}
\label{tab:veatic_arousal}
\end{table*}

\Cref{tab:veatic_arousal} reports arousal prediction on VEATIC.
Keeping the baseline setting unchanged and replacing only the naive mini-batch CCC term with B2P-Corr yields consistent gains across batch sizes.
At $B{=}4$, B2P-U improves CCC from $0.433$ to $0.450$ and sharply reduces RMSE from $0.311$ to $0.218$ (B2P-EMA: CCC $0.452$, RMSE $0.246$).
The improvements persist at $B{=}10$ (best CCC $0.578$ vs.\ $0.560$, RMSE $0.210$ vs.\ $0.230$) and at the default $B{=}20$, where B2P-U achieves the best overall CCC/PCC ($0.603/0.643$) with lower RMSE ($0.212$ vs.\ $0.234$).
Overall, B2P provides a practical and effective drop-in replacement for stabilizing CCC optimization on this sentimental video-based analysis dataset.

\subsection{Direction recovery and test set CCC for synthetic Exp. 2}
\label{app:exp2_ccc}

\Cref{fig:exp2_ccc} reports \textsc{DirectionRecovery} and CCC on the isotropic linear-regression test set for batch sizes $B\in\{2,4,8,16\}$. The overall behavior closely mirrors the PCC results in \Cref{fig:exp2}. In the extreme mini-batch regime ($B=2$), both B2P-EMA and B2P-U remain stable and continue to improve throughout training, reaching $\approx 0.85$ \textsc{DirectionRecovery} and $\approx 0.27$ CCC. Notably, this CCC is comparable to the naive method trained with a much larger batch size ($B=16$, $\approx 0.29$). In contrast, naive mini-batch CCC optimization degrades sharply at $B=2$, attaining only $\approx 0.18$ CCC and exhibiting substantially worse alignment with the ground-truth direction. Overall, Exp.~2 demonstrates that B2P-Corr enables reliable CCC maximization in the micro-batch setting, where the naive mini-batch objective suffers from severe batch-to-population bias and can become unstable.

\begin{figure*}[h]
  \centering
  \begin{subfigure}[t]{0.25\textwidth}
    \centering
    \includegraphics[width=\linewidth]{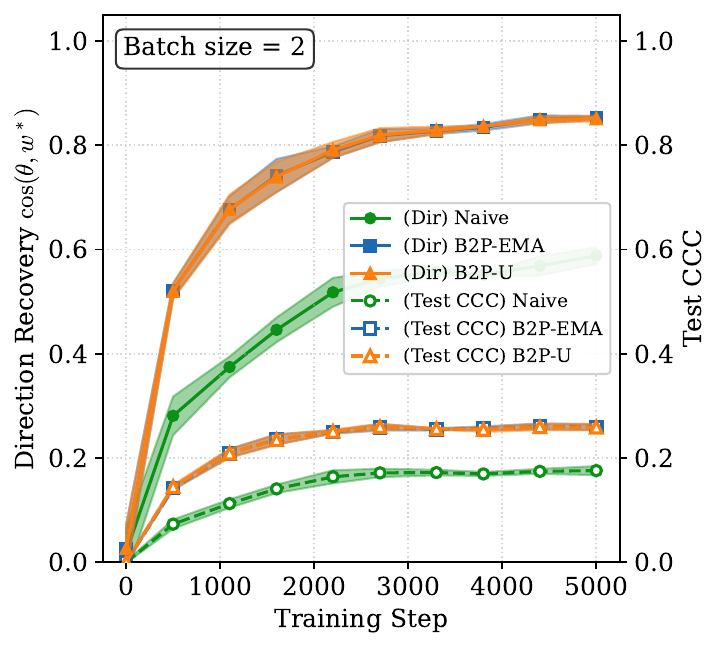}
    \caption{CCC: training ($B=2$)}
    \label{fig:exp2_ccc_a}
  \end{subfigure}\hfill
  \begin{subfigure}[t]{0.25\textwidth}
    \centering
    \includegraphics[width=\linewidth]{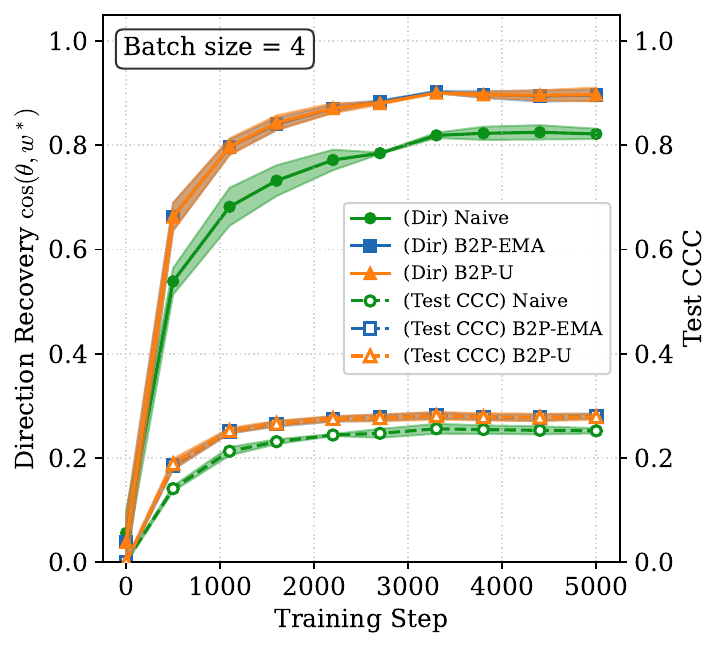}
    \caption{CCC: training ($B=4$)}
    \label{fig:exp2_ccc_b}
  \end{subfigure}\hfill
  \begin{subfigure}[t]{0.25\textwidth}
    \centering
    \includegraphics[width=\linewidth]{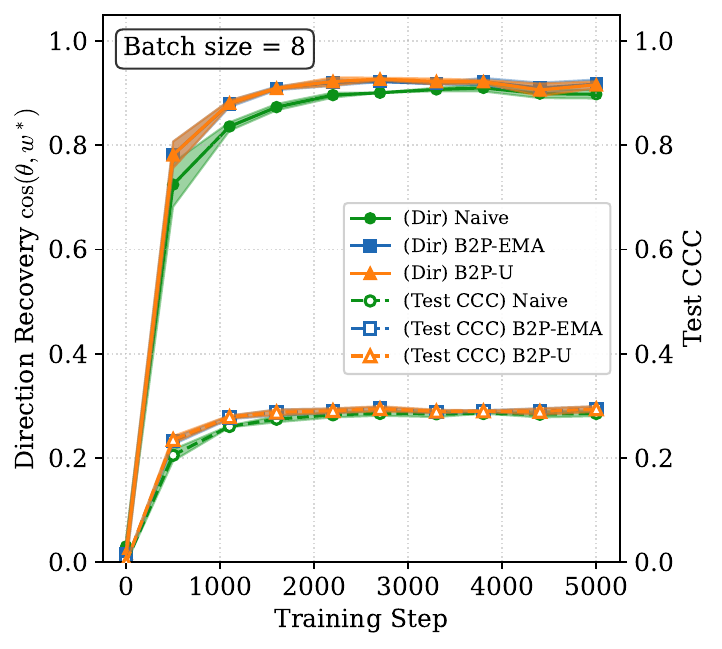}
    \caption{CCC: training ($B=8$)}
    \label{fig:exp2_ccc_c}
  \end{subfigure}\hfill
  \begin{subfigure}[t]{0.25\textwidth}
    \centering
    \includegraphics[width=\linewidth]{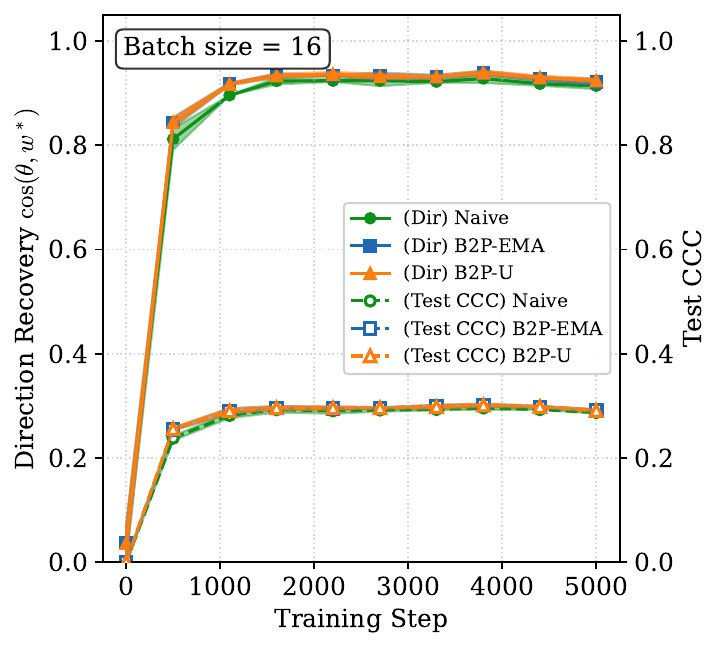}
    \caption{CCC: training ($B=16$)}
    \label{fig:exp2_ccc_d}
  \end{subfigure}
  \caption{For Exp.~2, we show DirectionRecovery and TestCCC vs. training steps for different batch sizes. With $B{=}2$, B2P-EMA/B2P-U shows largest increment comparing to naive mini-batch optimization. The gap narrows as $B$ increases.}
  \label{fig:exp2_ccc}
\end{figure*}

\subsection{B2P-EMA hyperparameter $\alpha$ sensitivity}
\label{app:alpha_sensitivity}

\Cref{fig:sensitivity_a,fig:sensitivity_b} report test-set PCC on KoNViD-1k for B2P-EMA with different smoothing factors $\alpha$ under $B{=}2$ and $B{=}8$. Overall, B2P-EMA is stable across hyperparameter choices. With small batches ($B{=}2$), smaller smoothing factors perform better: a smaller $\alpha$ updates the running moment sketch less aggressively, which is desirable because moment estimates from very small batches are noisier. Consequently, $\alpha{=}0.01$ outperforms $\alpha{=}0.2$ ($0.898$ vs.\ $0.893$). When increasing the batch size to $B{=}8$, moment estimates concentrate, and larger $\alpha$ values (e.g., $\alpha{=}0.10$ or $\alpha{=}0.20$) outperform $\alpha{=}0.01$ ($0.898$ vs.\ $0.893$). Overall, we recommend using a larger smoothing factor for larger batch sizes.

\begin{figure*}[h]
  \centering
  \begin{subfigure}[t]{0.25\textwidth}
    \centering
    \includegraphics[width=\linewidth]{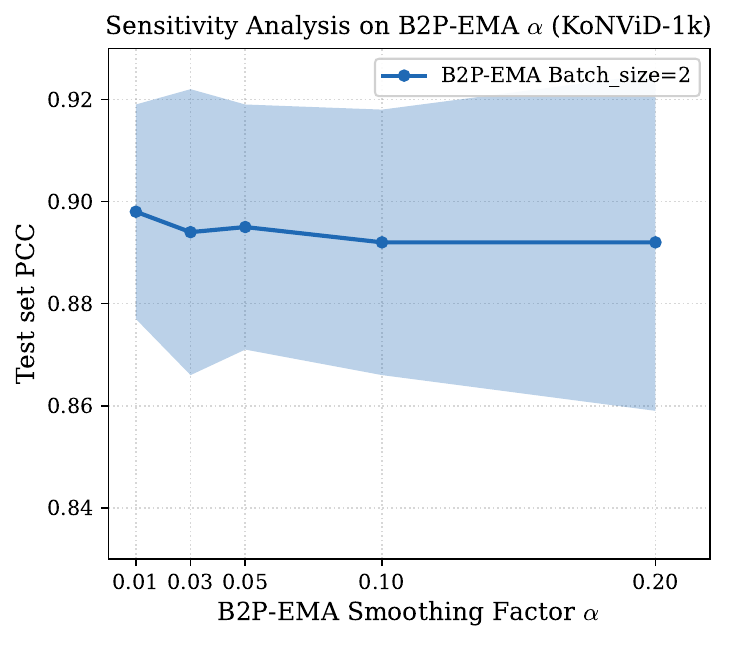}
    \caption{B2P-EMA $\alpha$ ($B=2$)}
    \label{fig:sensitivity_a}
  \end{subfigure}\hfill
  \begin{subfigure}[t]{0.25\textwidth}
    \centering
    \includegraphics[width=\linewidth]{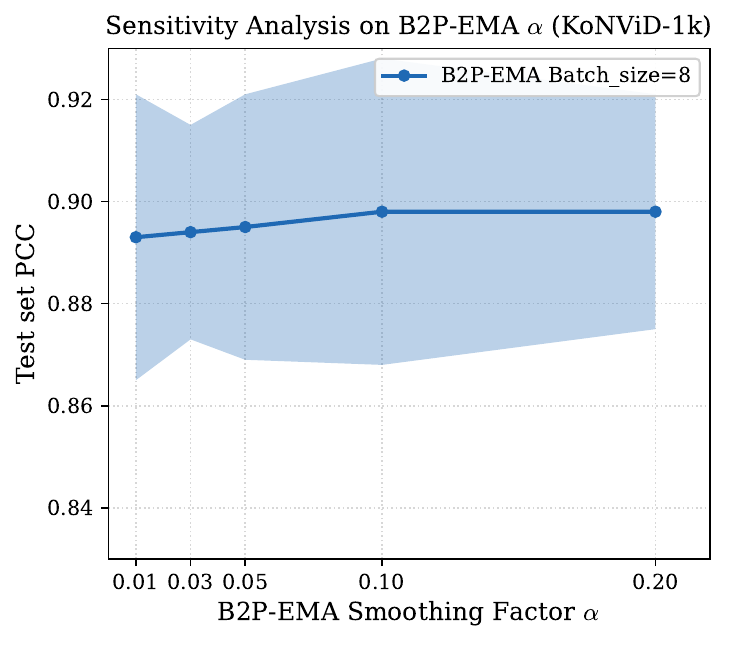}
    \caption{B2P-EMA $\alpha$ ($B=8$)}
    \label{fig:sensitivity_b}
  \end{subfigure}\hfill
  \begin{subfigure}[t]{0.25\textwidth}
    \centering
    \includegraphics[width=\linewidth]{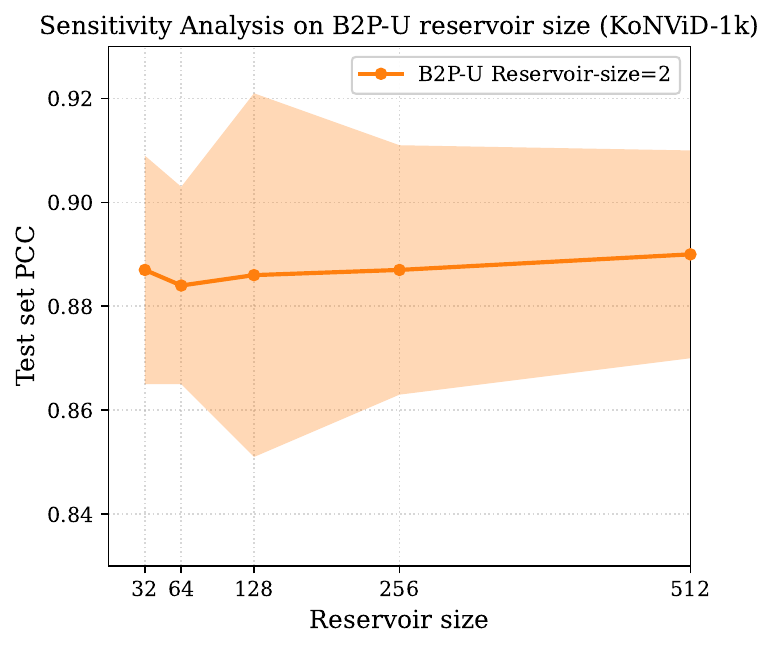}
    \caption{B2P-U Reservoir ($B=2$)}
    \label{fig:sensitivity_c}
  \end{subfigure}\hfill
  \begin{subfigure}[t]{0.25\textwidth}
    \centering
    \includegraphics[width=\linewidth]{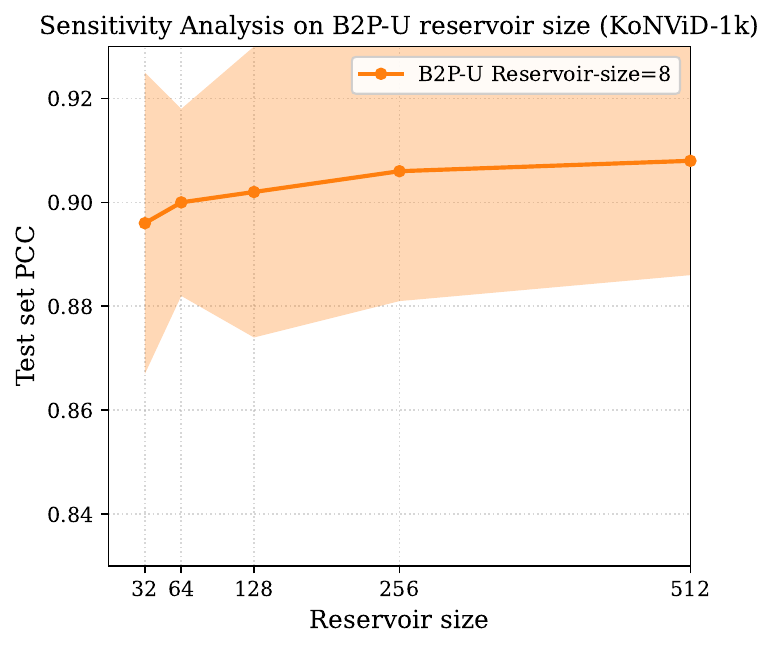}
    \caption{B2P-U Reservoir ($B=8$)}
    \label{fig:sensitivity_d}
  \end{subfigure}
  \caption{Sensitivity analysis on the KoNViD-1k dataset. We replace the naive mini-batch PCC objective with B2P-Corr and B2P-U, and evaluate (a--b) test-set PCC of B2P-EMA under different smoothing factors $\alpha$, and (c--d) test-set PCC of B2P-U under different reservoir sizes, for batch sizes $B\in\{2,8\}$.}
  \label{fig:sensitivity}
\end{figure*}

\subsection{B2P-U reservoir size sensitivity}
\label{app:reservoir-sensitivity}

\Cref{fig:sensitivity_c,fig:sensitivity_d} report test-set PCC on KoNViD-1k for B2P-U with different reservoir sizes under $B{=}2$ and $B{=}8$. B2P-U remains stable once the reservoir size exceeds $32$. For both $B{=}2$ and $B{=}8$, test-set PCC generally increases with reservoir size. However, larger reservoirs also increase the risk of stale cached predictions (if the reservoir is refreshed infrequently) and add computational overhead.

\end{document}